\pdfoutput=1
\documentclass[11pt]{article}

\usepackage[final]{acl}

\usepackage{times}
\usepackage{latexsym}
\usepackage{amssymb}
\usepackage[most]{tcolorbox}
\tcbuselibrary{breakable,skins}
\usepackage{parcolumns}
\usepackage{subcaption}
\usepackage[T1]{fontenc}
\usepackage{graphicx} %% added
\usepackage{xspace} %% added
\usepackage{multirow}

\usepackage[utf8]{inputenc}

\usepackage{microtype}

\usepackage{inconsolata}
\usepackage{amsmath}
\usepackage{amsthm}
\newtheorem{theorem}{Theorem}
\usepackage{longtable}
\usepackage{graphicx}
\usepackage{booktabs}
\usepackage{algorithm}
\usepackage{algpseudocode}
\newcommand{\sm}[1]{\textcolor{red}{Sourav: #1}}

\newtheorem{definition}{Definition}
\usepackage{mathtools}

\title{Colorful Talks with Graphs: Human-Interpretable Graph Encodings for Large Language Models}
\author{
  Angelo Zangari \quad
 Peyman Baghershahi\quad
 Sourav Medya\\
  University of Illinois Chicago, Chicago, IL, USA \\
  \texttt{\{azang,pbaghe2,medya\}@uic.edu}
}

\newcommand{\baselineOne}{TLG-A}
\newcommand{\baselineTwo}{TLG-F}
\newcommand{\ourmethod}{\textsc{CL-OWL}\xspace}
\newcommand{\lowl}{\textsc{\textsc{L-OWL}}\xspace}
\newcommand{\cowl}{\textsc{\textsc{C-OWL}}\xspace}
\definecolor{PromptGray}{RGB}{248,248,248}
\definecolor{SysBlue}{RGB}{35,99,177}
\definecolor{UserGreen}{RGB}{22,128,72}
\definecolor{GraphPurple}{RGB}{111,66,193}
\definecolor{WlOrange}{RGB}{230,126,34}
\definecolor{ColorTeal}{RGB}{0,128,128}
\definecolor{TargetRed}{RGB}{192,57,43}

\newcommand{\PromptSection}[3]{%
  \begin{tcolorbox}[%
    enhanced, breakable,
    colback=PromptGray,
    colframe=#1,
    boxrule=0.7pt, arc=2pt,
    left=6pt,right=6pt,top=6pt,bottom=6pt,
    colbacktitle=#1,
    coltitle=white,
    title=\textbf{#2},
    fonttitle=\normalsize,
  ]
  \ttfamily\footnotesize\obeylines\obeyspaces
  #3
  \end{tcolorbox}%
}

\newenvironment{FullWidthPromptFigure}[2]{%
  \def\FWPFcaption{#1}%
  \def\FWPFlabel{#2}%
  \begin{figure*}[t]
  \centering
  \begin{tcolorbox}[%
    enhanced,
    width=\textwidth,
    colback=white,
    colframe=black!25,
    boxrule=0.6pt,
    arc=3pt,
    left=8pt,right=8pt,top=8pt,bottom=8pt,
    colbacktitle=black!10,
    coltitle=black,
    title=\textbf{#1},
    fonttitle=\normalsize,
  ]%
}{%
  \end{tcolorbox}
  \caption{\FWPFcaption}
  \label{\FWPFlabel}
  \end{figure*}%
}

\begin{document}
\maketitle

\begin{abstract}
Graph problems are fundamentally challenging for large language models (LLMs). While LLMs excel at processing unstructured text, graph tasks require reasoning over explicit structure, permutation invariance, and computationally complex relationships, creating a mismatch with the representations of text-based models. Our work investigates how LLMs can be effectively applied to graph problems despite these barriers. We introduce a human-interpretable structural encoding strategy for graph-to-text translation that injects graph structure directly into natural language prompts. Our method involves computing a variant of Weisfeiler–Lehman (WL) similarity classes and maps them to human-like color tokens rather than numeric labels. The key insight is that semantically meaningful and human-interpretable cues may be more effectively processed by LLMs than opaque symbolic encoding. Experimental results across multiple graph algorithms and predictive tasks show significant improvements from our method on both synthetic and real-world datasets. By capturing both local and global-range dependencies, our method enhances LLM performance, especially on graph tasks that require reasoning over global graph structure.
%Code is available at: https://gitlab.com/zyx7713/abc5545.git

%\begin{itemize}
%    \item Finish abstract and intro
 %   \item Related work: need headings
 %   \item Figure 1 - improve
%    \item Section 4: shorten
%    \item Section 4: refer example and prompt
%    \item section 4: adjust theory
%    \item experiments
%     \item experiments
%      \item experiments
%      \item finish conclusion
%\end{itemize}
%Rapid rise of llms. Have seen applications and established as leading models in several areas (nlp, vision, code intelligence).
%While work has been done to apply them also to graph, results are not as good as other fields. 
%Due to inherent graph complexities, structure.
%Additionally wide variety of graph tasks, from node class to hard ones, combinatorial. Limitation is that this requires reasoning on graph, which is non trivial.
%Existing work done on this, but doesnt leverage graph structure to convey to llm. but that is necessary to llm.
%in this work we introduce a method to infuse graph structure into llm, method is general, so can be applied to any graph dataset. evaluation done across different graph types and tasks, and across llm generations. datasets both syntetically generated and also sampled from existing dataset widely used in graphml commyunity (cora, citeseer, pubmed, obgn-arxiv, ...)
\end{abstract}

\section{Introduction}

Large language models (LLMs) have advanced rapidly over the last few years, demonstrating substantial improvements in reasoning, tool use, and multimodal capabilities across successive generations. These advances have enabled LLMs to be deployed in domains traditionally dominated by specialized models, including program synthesis, scientific reasoning, and structured problem solving \cite{deepseekr12025}. Recently, the emergence of agentic LLM systems, capable of interacting with tools, APIs, and external environments via standardized protocols, has further blurred the boundary between general-purpose language models and domain-specific reasoning systems. Together, these developments raise a natural question: \textit{to what extent can LLMs serve as general solvers for structured reasoning tasks that have historically required specialized architectures?}

%Large language models (LLMs) have advanced rapidly in the last three years, demonstrating substantial improvements in reasoning, tool use, and multi-modal capabilities across successive generations. Recently, LLMs have been deployed in domains traditionally dominated by specialized models, and the emergence of agentic systems has further blurred the boundary between general-purpose language models and domain-specific reasoning systems. These developments raise a natural question: \textit{To what extent can LLMs replace or complement specialized architectures such as graphs for structured reasoning tasks?}

Graph problems provide a particularly revealing testbed for this question. Unlike text or images, graphs are combinatorial objects with no canonical ordering, and many graph tasks require permutation invariance, symmetry awareness, and reasoning over multi-hop relational structure \cite{wangMicrostructuresAccuracyGraph2024c}. Moreover, a large class of canonical graph problems (e.g., subgraph isomorphism) are NP-hard or NP-complete \cite{manchanda2020gcomb,ranjan2022greed}, placing substantial demands on systematic reasoning and abstraction. As a result, graph tasks expose a fundamental mismatch between graph structure and the sequential, text-based representations that LLMs are trained to process \cite{positiongfm2024}.

%Graph problems remain a notable challenge in this context. Unlike text or images, graphs are combinatorial objects with no canonical ordering, and many graph tasks require permutation invariance, symmetry awareness, and integration of both local attributes and global relational structure. Moreover, many canonical graph problems---such as graph coloring, subgraph isomorphism, and Hamiltonian path---are NP-hard or NP-complete, placing substantial demands on reasoning and abstraction. As a result, graph tasks expose a fundamental mismatch between graph structure and the sequential, text-based representations LLMs are trained to process.

Recent work has explored whether LLMs can nevertheless be applied to graph problems by translating graphs into text. Early studies show that LLMs can answer simple queries such as node degree, edge existence, and local neighborhood questions when graphs are serialized into natural language prompts \cite{guoGPT4GraphCanLarge2023,wangCanLanguageModels2024}. Subsequent work expanded this line of inquiry to include path finding, cycle detection, and limited forms of combinatorial reasoning \cite{fatemiTalkGraphEncoding2023,sanfordUnderstandingTransformerReasoning2024}. In parallel, several approaches have proposed augmenting graph-to-text translations with additional structure, such as similarity cues, embeddings, or symbolic annotations, to better guide LLM reasoning \cite{liuOneAllTraining2024b,wangMicrostructuresAccuracyGraph2024c,sunLargeLanguageModels2025}.

Despite this progress, existing works remain limited. Most studies focus on a simple set of tasks, typically emphasizing local or classification-style problems such as node classification or link prediction. Evaluations are often restricted to small graphs, leaving open questions about scalability and task generality. Furthermore, as closed-source LLMs cannot be retrained on graph data, most methods rely on prompt engineering or representational choices where theoretical grounding is unclear. In particular, many graph-to-text encodings introduce symbols or numeric identifiers that are structurally meaningful but linguistically opaque.

\paragraph{Our Contributions.} Our work is motivated by the hypothesis that human-interpretable structural representations can help bridge the gap between graph structure and LLM reasoning. LLMs are pretrained and aligned on vast corpora of human language and are therefore biased toward representations that resemble human explanatory conventions. We posit that graph structure, when expressed through interpretable and semantically grounded cues, may be more effectively exploited by LLMs than arbitrary numeric or symbolic encodings. Our contributions span representation design, theoretical motivation, and rigorous experiments with various tasks and data. They are as follows.

\noindent\textit{(i) Human-interpretable structural encoding}: 
We introduce a graph-to-text encoding strategy that injects explicit structural information into LLM prompts using human-interpretable tokens derived from Weisfeiler-Leman (WL) refinement. This proposed representation preserves permutation invariance and structural similarity while aligning graph structure with linguistic abstractions familiar to pretrained LLMs.

\noindent
\textit{(ii) Structure-preserving representation}: We provide a principled analysis connecting ordered WL refinement to distance-weighted notions of node connectivity. Under mild assumptions, we show that the induced WL ordering is consistent with a broad class of centrality-like measures.

\noindent
\textit{(iii) Experiments}: We conduct a comprehensive empirical study across diverse graph families and task types to assess the effectiveness of human-interpretable structural encoding.

\section{Related Work}

%LLMs have rapidly become a general-purpose tool for a wide range of tasks that traditionally required specialized algorithms or domain-specific models. As their capabilities expanded, researchers began to investigate whether LLMs could also be used to solve problems defined on graphs, structures that encode relationships, connectivity, and higher-order patterns. This line of work is motivated by the observation that LLMs can process arbitrary text, and that graphs can, in principle, be translated into textual form. If this translation is effective, LLMs could potentially act as drop-in solvers for graph queries without requiring model-specific retraining.

%We begin with the methods which are motivated by the observation that LLMs can process arbitrary text, and that graphs can be translated into textual form. %If this translation is effective, LLMs could potentially act as solvers for graph queries without requiring retraining. 
We begin with methods based on the idea that graphs can be represented as text and processed by LLMs. Initial studies \cite{fatemiTalkGraphEncoding2023,guoGPT4GraphCanLarge2023,wangCanLanguageModels2024} explored basic tasks such as node degree, edge existence, and simple structural queries. Some have broadened the scope to include limited forms of combinatorial reasoning. LLMs can often answer small or local queries reliably, but performance tends to degrade as graphs become larger or tasks require multi-step reasoning \cite{fatemiTalkGraphEncoding2023}. %Despite these limitations, the possibility of using generic LLMs for graph reasoning has made this an active and  growing research area.

%\subsection{Different Applications of LLMs to Graph Tasks}
\textit{Different Ways of Applying LLMs on Graphs.}
Another set of approaches provide LLMs an additional structure or guidance, such as augmenting graph descriptions with similarity cues, embeddings, or human-interpretable features. Prior work shows that LLMs can play several distinct roles within graph-related pipelines.  Existing methods fall broadly into three categories: \textit{(i) Solvers}: LLMs directly compute answers to tasks such as degree queries, path finding, triangle counting, etc., using only text prompts \cite{wangCanLanguageModels2024,guoGPT4GraphCanLarge2023,fatemiTalkGraphEncoding2023,zhangLLM4DyGCanLarge2024,sunLargeLanguageModels2025}. \textit{(ii) Aligners}: LLMs guide or interface with graph-learning models, enforce constraints, or translate natural-language instructions into graph-processing steps \cite{jinGraphChainofThoughtAugmenting2024,zhuParameterEfficientTuningLarge2024,tanGraphorientedInstructionTuning2025,liAreLargeLanguage2025}. \textit{(iii) Encoders}: LLMs convert graphs into detailed textual representations, embeddings, or rationales that downstream modules or humans can interpret \cite{zhaoGraphTextGraphReasoning2023b,fatemiTalkGraphEncoding2023,liuOneAllTraining2024b,zhuLLMGNNGraph2025}. Please refer to these recent surveys \cite{jinLargeLanguageModels2024,renSurveyLargeLanguage2024c}.

Recent representative systems instantiate these directions in different ways. GraphText \cite{zhaoGraphTextGraphReasoning2023b} performs graph reasoning in text space by constructing graph-aware textual contexts for LLM inference. OFA \cite{liuOneAllTraining2024b} instead trains a unified graph model across classification tasks, and LLaGA \cite{chenLLaGALargeLanguage2024c} combines language models with graph-side representations through a graph assistant architecture. These methods show that graph structure can benefit from textual, learned, or hybrid graph-language representations. Our work is complementary: rather than training a graph-language model or injecting continuous embeddings, we study a training-free graph-to-text encoding based on ordered WL descriptors and natural-language color tokens. This lets us isolate how discrete, human-interpretable structural cues affect LLM graph reasoning under controlled prompts.

%Solvers: LLMs directly compute answers to tasks such as degree queries, pathfinding, triangle counting, etc., using only text prompts \cite{wangCanLanguageModels2024,guoGPT4GraphCanLarge2023,fatemiTalkGraphEncoding2023,zhangLLM4DyGCanLarge2024,sunLargeLanguageModels2025}.

%Aligners: LLMs guide or interface with graph-learning models, enforce constraints, or translate natural-language instructions into graph-processing steps \cite{jinGraphChainofThoughtAugmenting2024,zhuParameterEfficientTuningLarge2024,tanGraphorientedInstructionTuning2025,liAreLargeLanguage2025}.

%Encoders: LLMs convert graphs into detailed textual representations, embeddings, or rationales that downstream modules or humans can interpret \cite{zhaoGraphTextGraphReasoning2023b,fatemiTalkGraphEncoding2023,liuOneAllTraining2024b,zhuLLMGNNGraph2025}.

%\subsection{Using LLMs as graph task solvers}

\textit{Methods as Solvers using LLMs.}
A growing line of work examines LLMs used directly as solvers for graph queries.  These methods vary along three orthogonal dimensions such as graph types, graph tasks, and prompt formats. Many solver-oriented studies evaluate LLMs on real-world text-attributed graphs such as Cora, Citeseer, PubMed, and the OGB datasets \cite{yeLanguageAllGraph2024,qinDisentangledRepresentationLearning2024,liuOneAllTraining2024b,heUniGraphLearningUnified2025}. In parallel, several benchmarks for explicit graph-in-text reasoning rely on synthetic graphs with controlled topology, such as NLGraph or G-Recall \cite{wangCanLanguageModels2024,wangMicrostructuresAccuracyGraph2024c,sanfordUnderstandingTransformerReasoning2024}. These datasets offer meaningful semantic labels, but their size and irregular structure introduce challenges for text-based models. Beyond these, a substantial number of works also experiment with synthetic graphs because their structure is easier to control and analyze \cite{wangCanLanguageModels2024,wangMicrostructuresAccuracyGraph2024c,sanfordUnderstandingTransformerReasoning2024}. When graphs are fully serialized into a single prompt, evaluations almost always restrict to small graphs since LLM performance declines quickly as the number of nodes or edges grows.

%each graph instance to small sizes, typically 5–30 nodes, since LLM performance declines quickly as the number of edges grows.

%\textbf{Scope and limitations of previous work.}
%Other existing works focus on a wide spectrum of graph tasks: local structural queries (degree, neighbors, edge existence); mid-level tasks (connectivity, cycle detection); and more demanding algorithmic problems (shortest paths, triangle counting, subgraph detection). Early studies concentrated on classification-style tasks, where the model labels nodes or predicts edge existence \cite{guoGPT4GraphCanLarge2023,yeLanguageAllGraph2024,qinDisentangledRepresentationLearning2024,liuCanWeSoft2024c,liuOneAllTraining2024b}. More recent papers probe LLMs with combinatorial queries requiring multi-step reasoning \cite{wangCanLanguageModels2024,fatemiTalkGraphEncoding2023,zhangLLM4DyGCanLarge2024,sanfordUnderstandingTransformerReasoning2024,wangLLMsZeroshotGraph2024}. Across these settings, performance decreases sharply when tasks involve global coordination or systematic search. This limitation is particularly relevant for our work, which targets tasks where structure must be understood rather than memorized.

Additional details on related work have been provided in Appendix \ref{app:rel-work}.
%Regarding the process of the graph to text translation itself, previous work follows four not-mutually-exclusive approaches, as reported in appendix \ref{app:rel-work}.

%\newpage
%\section{Our Approach \sm{writing For paper}}

\begin{figure*}[h!]
    \centering
    \includegraphics[width=\textwidth]{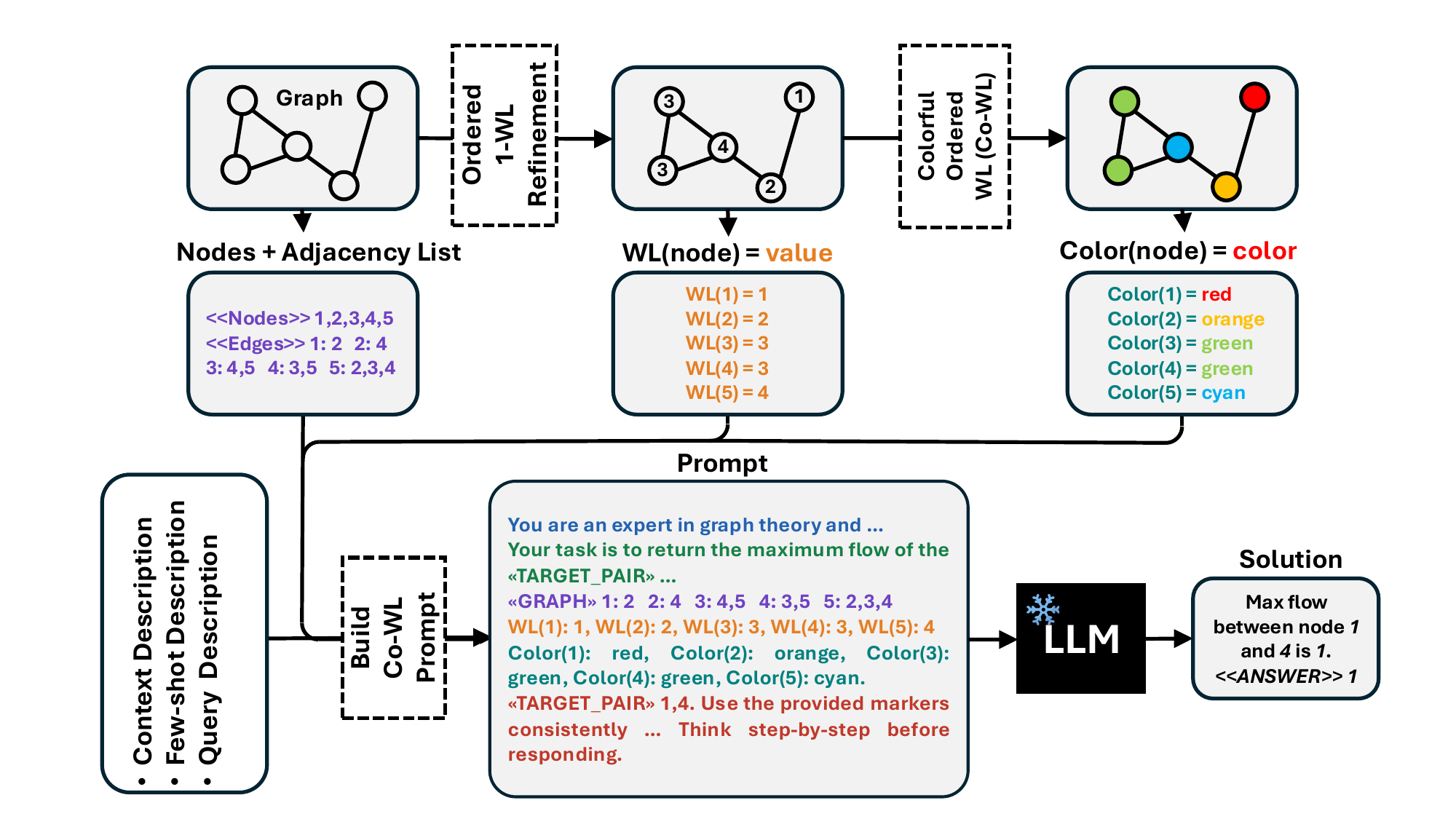}
    \caption{Overview of the proposed Colorful Ordered Weisfeiler-Leman (\ourmethod) pipeline. A graph is first represented as nodes with adjacency lists, then processed using ordered 1-WL refinement to compute node-level structural labels. These labels are mapped to human-tangible color tokens and embedded into a natural-language prompt. The resulting prompt enables an LLM to perform graph reasoning tasks (e.g., maximum-flow).}
    \label{fig:pipeline}
    \vspace{-2mm}
\end{figure*}
%\az{where should we place system diagram?}

\section{The Problem Setup}
Unlike GNNs, LLMs are sequence models: they consume and produce text, and have no native interface for graphs. Using an LLM to solve a graph problem therefore requires a \emph{graph-to-text translation} step. The graph, the task description, and any auxiliary information must be encoded as a textual sequence. The central objective of our work is to design a principled graph-to-text representation method that enables LLMs to reason about graphs.%, including those requiring non-trivial structural understanding.

\textbf{Formulation.} Solving a graph problem with an LLM requires transforming a structured, non-sequential object into a textual sequence. We formalize the elements involved in this transformation and highlights the key factors that influence overall effectiveness. Let $G = (V, E)$ be a graph with node set $V$ and edge set $E \subseteq V \times V$, and let $q \in Q$ denote a graph task (e.g., cycle detection, reachability between two nodes). LLMs operate over a set of text sequences $W \subset \Sigma$ where $\Sigma$ is a finite alphabet, and can be viewed as a function $f_\theta : W \rightarrow W'$ where $W' \subset \Sigma$ , parameterized by $\theta$. 

We translate the graph and task into a textual prompt via a translation function $\tau : X \rightarrow W$, where $X$ encompasses all information supplied to the model: graph structure, auxiliary structural information, task description, and prompt scaffolding. This transformation may not be unique. Subsequently, for an input query $q$ (a graph task) the model prediction becomes:
\begin{equation}
    y = f_\theta\big(\tau(G, q, x_{\text{struct}})\big).
\end{equation}

%This formulation enables the LLM to only processes the textual representation constructed by $\tau$ instead of the graph directly. \pb{This is the only way and obvious; otherwise, the LLM cannot be fed by the graph. So you can remove this sentence. }
Based on this formulation, the overall effectiveness of a method depends on two major components: \textbf{(1) Information content} $x_{\text{struct}}$:  the structural information from the graph.  
    This may range from an edge list to richer structural summaries.  \textbf{(2) Translation strategy} $\tau$:  
    how the information is expressed in text.  
    For example, one can include ordering of edges, natural-language descriptions, few-shot examples, explicit markers for labels or colors, and the level of verbosity. %\textbf{(3) Model choice} $f_\theta$:  
    %the specific LLM, given substantial differences in reasoning ability, context handling, and robustness across model classes. \pb{This is repeated. Remove the previous part talking about this. Read the comments I put above for that section. If you are talking about the effectiveness of the model, then ``Model Choice'' could be part of it, but still it's obvious information. You can say: ``Given a model $f_{\theta}$. the performance depends on ...}
These components interact strongly. For example, adding more structural information improves expressiveness but may increase prompt length and complexity; likewise, stronger models may benefit more from richer encoding. %\todo{refer to experiments} %\pb{Do we have any analysis on how the performance changes by changing the model and why? If we don't have that, remove after ``likewise.''}
%\az{on the \textbf{how} yes, on the \textbf{why} no}

%\az{maybe add note that also llm $f_\theta$ can be varied but outside scope of method - although obviously important and interesting for experimental results}

\label{sec:method}
\section{Our Proposed Method}

%\todo{merge two versions of section 4 summary}

%\subsection*{Overview}
We introduce node-level structural identifiers based on Weisfeiler-Leman (WL) \cite{wlmain2011} refinement and map into a human-interpretable form using colors. Figure \ref{fig:pipeline} shows the overview of our method. %We show that our refined WL labels have connections with node centrality.

\subsection{Node-level Structural Identifiers ($x_{\text{struct}}$)}
\label{sec:method_structural_augmentation}

Many graph problems require reasoning beyond local neighborhood. For example, cycle detection, shortest paths, or flow-related tasks depend on multi-hop neighborhoods, symmetries, and recurring substructures. Therefore, identifiers that are created from global structural information enable better graph comprehension, and thus, becomes beneficial for all downstream tasks. Graph learning models (e.g., GNNs) aggregate a node's neighborhood information through message passing to generate structure-aware embeddings, such that similar nodes attain similar embeddings. However, such latent embeddings have no meaning to a language model and appear as arbitrary numerical values.

% Many graph problems require reasoning beyond local adjacency. Tasks such as cycle detection, reachability, shortest paths, or flow-related queries depend on non-local structural patterns, including multi-hop neighborhoods, symmetries, and recurring substructures. When solving such problems with LLMs, this structural information must be made explicit: the model cannot reliably infer it from a raw edge list alone, especially as graph size grows or structural patterns become less obvious.
% In classical graph machine learning, structural information is captured through learned representations. Graph neural networks, for example, compute node embeddings by aggregating information from neighboring nodes through message passing. These embeddings encode structural similarity: nodes occupying similar roles in the graph tend to have similar representations. However, such latent vectors are learned in a model-specific embedding space and have no intrinsic meaning outside the generating architecture.

To inject structural information, some prior approaches directly use the continuous GNN embeddings into the LLM prompts for graph problems \cite{gspell2025, chenLLaGALargeLanguage2024c}. Other methods replace node identifiers with alternative vocabularies or descriptive placeholders \cite{fatemiTalkGraphEncoding2023}. These strategies suffer from a common limitation. The injected symbols do not carry interpretable information for the LLM. Floating-point embeddings are tied to a specific model and training process, while renamed node identifiers remain arbitrary tokens whose relationships are not grounded in the model’s linguistic prior knowledge.
% When LLMs are applied to graph problems, prior work has attempted to transfer this idea by injecting structural information into the prompt. Some approaches include embedding continuous GNN representations directly in text form \cite{gspell2025, chenLLaGALargeLanguage2024c}, while others replace node identifiers with alternative vocabularies or descriptive placeholders \cite{fatemiTalkGraphEncoding2023}. These strategies suffer from a common limitation: the injected symbols do not carry interpretable meaning for the LLM. Floating-point embeddings are tied to a specific model and training process, while renamed node identifiers remain arbitrary tokens whose relationships are not grounded in the model’s linguistic prior.\pb{Are there previous works that use the exact embedding floating values? In the prompt?e}

% par. introduction and motivation of our general $x_{aug}$
To address these limitations, we construct node-level structural descriptors that: 1) are \emph{structure-aware} to allow capturing global topological information, 2) preserve \emph{structural similarity} to assign identical descriptors to nodes that are structurally similar, 3) acknowledge an \emph{ordered interpretation} to enable comparison between nodes based on the extent of their connectivity. These descriptors can be reliably used by LLMs.%These properties are advantageous for a kind of graph translation that an LLM can reliably exploit.

% Our goal is therefore to construct a node-level structural descriptor that satisfies three requirements. First, it should be \emph{structure-aware}, capturing multi-hop topological information rather than purely local features. Second, it should preserve \emph{structural similarity}, assigning identical descriptors to nodes that occupy equivalent structural roles. Third, it should admit an \emph{ordered interpretation}, enabling comparison between nodes based on the strength or extent of their structural connectivity. These properties are essential for downstream translation into a form that an LLM can reliably exploit.

\paragraph{Constructing a structural descriptor.}
%Weisfeiler--Lehman refinement as a structural descriptor.}
We adapt or refine the Weisfeiler-Leman (WL) \cite{wlmain2011} coloring algorithm which is a heuristic for graph isomorphism test. The algorithm iteratively refines node labels by aggregating information from neighborhoods. At each iteration, a node updates its label based on its current label and the multiset of labels of its neighbors. After $k$ iterations, the resulting label summarizes the node’s $k$-hop neighborhood structure. %\sm{we could give an example}%\pb{I modified this paragraph. But replace the whole paragraph with a technical formulation of the WL method instead of a verbal definition.}

% \paragraph{Weisfeiler--Lehman refinement as a structural descriptor.}
% The Weisfeiler--Lehman (WL) color refinement algorithm provides a natural starting point. Originally introduced as a heuristic for graph isomorphism testing, WL iteratively refines node labels by aggregating information from local neighborhoods. At each iteration, a node updates its label based on its current label and the multiset of labels of its neighbors. After $k$ iterations, the resulting label summarizes the node’s $k$-hop neighborhood structure \cite{wlmain2011}.

% par. meaning associated to wl labels
%WL is beneficial for our purpose as it is task- and input-agnostic, producing node-level descriptors that reflect structural similarity. Nodes with identical WL labels are indistinguishable and thus have equivalent structural information. However, 
WL labels in their standard form are poorly suited for direct use in LLMs. The labels are arbitrary integers produced by hashing operations; their numeric values have no semantic meaning, no inherent ordering, and no interpretable relation to each other. This misalignment is particularly problematic for LLMs that are pre-trained on natural language patterns and struggle with reasoning over such arbitrary integers or opaque identifiers.

To address the above limitation, we modify the WL method to induce an ordering over node labels that are informative for LLMs. The key idea is to constrain the aggregation step so that label updates reflect not only structural similarity but also a notion of cumulative structural connectivity.

% To address the above limitation, we modify the WL refinement \pb{Why do you call it refinement?} to induce an interpretable ordering over node labels. The key idea is to constrain the aggregation step so that label updates reflect not only structural equivalence but also a notion of cumulative structural connectivity.

%%%%%%%%%%%%%---Specifically, during aggregation, we sort incoming messages first by increasing neighborhood size and then lexicographically before the label digest \sm{didn't understand the second part of the sentence}. This deterministic ordering ensures that nodes with comparably larger or more complex neighborhoods receive labels that are consistently ordered over successive WL iterations. Hence, the final labels provide better structural encoding that goes beyond simple degree and acknowledges global graph context through iterative refinement. %\pb{I revised the paragraph but it's repeated in the ``Design Objective'' paragraph.}

% Specifically, during aggregation, we sort incoming messages first by increasing neighborhood size and then lexicographically before computing the new label digest. This deterministic ordering ensures that nodes whose neighborhoods grow larger or more complex over successive iterations receive labels that are consistently ordered relative to less connected nodes. Under these constraints, the final labels can be interpreted as encoding a generalized notion of structural connectivity that extends beyond simple degree while still incorporating global graph context through iterative refinement.

%\sm{adding this: p, please check...}
\subsubsection*{Ordered 1-WL refinement}
We compute node-level structural descriptors using a refinement procedure of the canonical
1-dimensional Weisfeiler-Leman (1-WL)  \cite{wlmain2011}. Let us initialize node labels $\ell_v^{(0)} \in \mathbb{N}$ (e.g., $\ell_v^{(0)}=1$ for all $v\in V$). Let $\mathcal{N}(v)$ be the self-exclusive neighborhood function s.t. $\mathcal{N}(v)=\{u: (v, u) \in E\}$. For iterations $t=0,1,\dots,T-1$, update each node label by aggregating the multiset of
neighbor labels and applying a deterministic, order-preserving canonicalization:
\begin{equation*}
\ell_v^{(t+1)} \;=\; \mathrm{ID}\!\left(
\ell_v^{(t)},\;
\operatorname{Sort}\!\left(\{\!\{\,\ell_u^{(t)} : u \in \mathcal{N}(v)\,\}\!\}\right)
\right)
\label{eq:ordered-1wl}
\end{equation*}
Here $\{\!\{\cdot\}\!\}$ denotes a multiset and $\operatorname{Sort}(\cdot)$ converts it to a tuple that is sorted lexicographically (thus sorted in ascending numerical order and preserving multiplicities). The function $\mathrm{ID}(\cdot)$ assigns
a unique integer identifier to each distinct pair
$\big(\ell_v^{(t)}, \operatorname{Sort}(\cdot)\big)$ \emph{in a globally deterministic order},
by (i) collecting all node messages
$m_v^{(t)}=\big(\ell_v^{(t)}, \operatorname{Sort}(\{\!\{\ell_u^{(t)}:u\in\mathcal{N}(v)\}\!\})\big)$,
(ii) sorting the set of unique messages lexicographically, and (iii) mapping the $i$-th message
in this sorted list to identifier $i$. These labels form our node-level structural prompt augmentation $x_{\text{struct}}$.

\subsubsection*{Ordered 1-WL \& Centrality}
We formalize the relationship between our \emph{ordered} 1-WL refinement and standard notions of node centrality. We define centrality as follows.

\begin{definition}[Distance-shell counts and truncated connectivity]
For $k> 0$, define the $k$-th distance shell of $v$ as
$S_k(v)=\{u\in V : d(u,v)=k\}$, where $d(\cdot,\cdot)$ is shortest-path distance.
For a weight sequence $\alpha_0\ge \alpha_1\ge \dots \ge \alpha_T > 0$, define the truncated
distance-weighted connectivity
\[
C_T(v) \;=\; \sum_{k=0}^{T} \alpha_k\,|S_k(v)|.
\]
\end{definition}

\begin{theorem}%[Local consistency with degree and multi-hop neighborhood growth]
\label{thm:wl-centrality}
Let $\ell^{(t)}$ be the labels produced by ordered 1-WL on $G$.

\smallskip
\noindent\textbf{(1) Degree consistency.}
For any nodes $v,w\in V$,
$\deg(v)>\deg(w) \;\Longrightarrow\; \ell_v^{(1)} > \ell_w^{(1)}$.

\smallskip
\noindent\textbf{(2) Shell-dominance implies label dominance.}
Fix $T\ge 1$. Suppose there exist nodes $v,w\in V$ such that 
(i)$|S_k(v)| \;\ge\; |S_k(w)| \text{ for all } k=1,\dots,T,$ and (ii)
$|S_{k^\star}(v)| \;>\; |S_{k^\star}(w)| \text{ for some } k^\star\le T,$
and additionally the rooted $T$-hop neighborhoods of $v$ and $w$ are \emph{tree-unfoldings}
(i.e., have no collisions in the breadth-first expansion up to depth $T$).
Then $\ell_v^{(T)} \;>\; \ell_w^{(T)}
$ holds.

\smallskip
\noindent\textbf{(3) Relation with connectivity.}
With (2), for any nonincreasing positive weights $\{\alpha_k\}_{k=0}^T$,
$C_T(v) \;>\; C_T(w).$ Hence, ordered 1-WL induces an ordering that is consistent with a class of distance-weighted
connectivity functions on locally tree-like neighborhoods.
\end{theorem}

The proof is given in Appendix \ref{app:proof}. Theorem~\ref{thm:wl-centrality} provides a connection between ordered 1-WL labels and
distance-weighted connectivity. This shows the benefits of using 1-WL refinement as structural identifiers. 
% \sm{OLD para}We employ the one-dimensional WL refinement (1-WL) \sm{cite 1-WL} with these ordering constraints. The resulting labels retain the WL's similarity-preserving property while providing an ordered interpretation that reflects increasing structural complexity. These labels form our node-level structural prompt augmentation $x_{\text{struct}}$, that is fully deterministic, task-agnostic, and depends only on the input graph. In addition, the WL labels have a \textit{monotonic structural scale}, i.e., the labels have structural meaning, and lower graph connectivity implies lower label values, making the label scale monotonic with respect to connectivity \sm{give an example}. %\pb{Explained the monotonic structural scale based on how x refers to this property later. Not sure if the definition is known with a different term in the literature.}

% We employ the one-dimensional Weisfeiler--Lehman refinement (1-WL) with these ordering constraints, iterating until label stabilization. The resulting labels retain the classical WL property of structural equivalence while additionally admitting an ordered interpretation that reflects increasing structural complexity. These ordered WL-derived labels form our node-level structural augmentation $x_{\text{struct}}$.

\subsection{Translating Structural Information into Text ($\tau$) \& \ourmethod Prompting}
\label{sec:method_translation}

To make $x_{\text{struct}}$ usable by an LLM, it must be translated into text via the translation function $\tau(G, q, x_{\text{struct}})$. As discussed before, arbitrary integers or symbolic codes are generally suboptimal for conveying structural meaning to LLMs. %Such tokens lack semantic grounding in natural language and provide limited cues for relational reasoning. 
Meanwhile, LLMs are effective at leveraging representations that resemble natural explanatory language. An effective translation strategy should therefore express structural information in a form that is both faithful to the underlying graph properties and suitable to exploit the linguistic priors of LLMs.

% The structural augmentation $x_{\text{struct}}$ defined in the previous section produces ordered, node-level labels that summarize graph structure. To make this information usable by an LLM, it must be translated into text via the translation function $\tau(G, q, x_{\text{struct}})$. This step is critical: LLMs cannot directly exploit latent representations or abstract numeric encodings unless they are expressed in a form aligned with the model’s training distribution.

% Although LLMs are capable of processing numerical tokens, arbitrary integers or symbolic codes are generally a suboptimal medium for conveying structural meaning. Such tokens lack semantic grounding in natural language and provide limited cues for relational reasoning. In contrast, LLMs are extensively trained on descriptive, human-authored text across many domains and are particularly effective at exploiting representations that resemble natural explanatory language. An effective translation strategy should therefore express structural information in a form that is both faithful to the underlying graph properties and compatible with the linguistic priors of the model.
% \pb{LLMs can generally interpret the data, even if it's numerical, but this method is suboptimal. So wirte the second chalenge again and also add: Since LLMs are trained with massive natural language data of different domains, an optimal method must exploit this information effectively}.

%\paragraph{Design objective.}
%Let $L_v$ be the derived WL labels. 
Our objective is to map these labels into a textual representation that preserves key properties: similarity-preserving, order-preserving, and conciseness. Nodes with similar structural features should have similar representations. For example, nodes with higher structural connectivity should be distinguishable from those with lower connectivity in a consistent and interpretable manner. %\pb{Revised the paragraph but it's repeated with a paragraph before---check the above comments.}

% Given the ordered WL-derived labels $L_v$ \pb{First time introducing the $L(.)$ function}, our objective is to map them into a textual representation that preserves two key properties: similarity and order. Nodes with identical or similar structural roles should be described in similar terms, while nodes with higher structural connectivity should be distinguishable from those with lower connectivity in a consistent and interpretable way. At the same time, the representation should remain simple, compact, and easily integrated into a textual prompt.

\paragraph{Colors as an interpretable similarity space.}
To achieve the above, we map WL-derived labels into colors. Colors provide a natural similarity space that is widely used in human communication to group and distinguish entities. In natural language, colors carry intuitive relational meaning (e.g., “red is closer to orange than blue”), and LLMs are well-equipped to exploit such associations. %\sm{can we cite any paper on LLM and colors?} 

% To meet these requirements, we map WL-derived labels into colors and express these colors explicitly in text. Colors provide a natural similarity space that is widely used in human communication to group and distinguish entities. In natural language, colors carry intuitive relational meaning (e.g., “red is closer to orange than to blue”), and LLMs are well-equipped to exploit such associations.

From a representational perspective, colors also admit an ordered interpretation. By restricting attention to a one-dimensional segment of the color spectrum, variations in hue can be associated with a monotonic structural scale derived from the ordered WL labels. This allows us to encode both equivalence and graded similarity in a way that remains human-interpretable. Figure \ref{fig:fullpagefigure} shows an example. %Please see the next section (Colorful-WL Structure descriptor) for details.

\begin{figure*}[h!]
\vspace{-2mm}
    \centering
    \includegraphics[width=.9\textwidth]{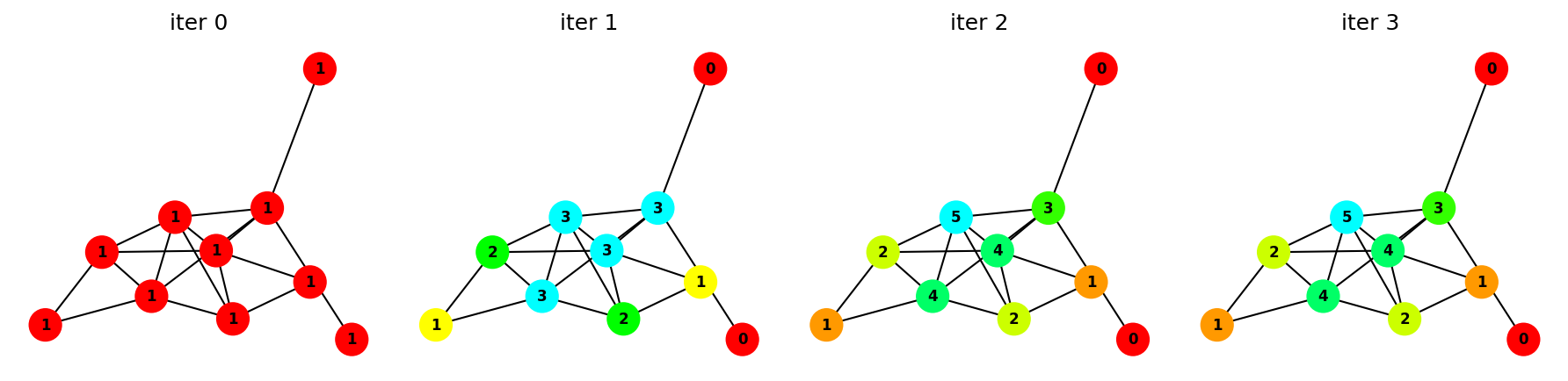}
    \caption{Iterative refinement of WL labels and their mapping to human-interpretable colors, showing how structural equivalence classes evolve as increasingly larger neighborhoods are incorporated.}
    \label{fig:fullpagefigure}
    \vspace{-2mm}
\end{figure*}

\subsubsection*{Colorful-WL Prompting (\ourmethod)}
Here we formally explain how our \textbf{Colorful-ordered-WL (\ourmethod)} translator $\tau$ constructs the LLM prompts. The main input to the LLM is a graph task $T \subset U$, where $U \in V \times E \times \mathbb{R}$ is a general set of graph problems/tasks, e.g. shortest path prediction. Each task is a set of tuples of query $q$ and answer $a$ as $T = \{(Q, S(V_s, E_s), a): S(V_s, E_s) \in V \times E, \, a\in \mathbb{R}, \, Q\subset V_s\}$. Note that $S(V_s, E_s)$ is a graph and not necessarily a subgraph of $G$, but we have $V_s \in V$ and edges $E_s \in E$. Also, $Q$ could be a singleton set, $\{u\} \subset V_s$, (e.g. in node classification) or a subset of nodes, $\{u_i\}_{i=1}^{m} \subset V_s$, (e.g., shortest path where $m=2$). Given a tuple $T^{(i)}=(Q^{(i)}, S(V_s^{(i)}, E_s^{(i)}), a^{(i)}) \in T, \, \forall i\in [|T|]$, our translator incorporates multiple channels of information in the prompts as follows:

\paragraph{1. Context descriptor.} A descriptor $\tau_c: W\times V \times E \rightarrow W$ uses the input subgraph $S(V_s^{(i)}, E_s^{(i)})$ along with prespecified graph notation, task definition, and LLM characterization templates to provide a rich description as the task descriptor generates a textual sequence $P_c^{(i)}=\tau_c(S(V_s^{(i)}, E_s^{(i)}))$.

\paragraph{2. Colorful-WL Structure descriptor.} A descriptor $\tau_{\text{cwl}}$ leverages the structural information collected from $S(V_s^{(i)}, E_s^{(i)})$ into the prompt. Specifically, let $WL(.)$ be the WL mechanism which generates labels $L_u = WL(u) \in \mathbb{Z}^+$ for $\forall u \in V_s^{(i)}$, and let $\bar{L}_u \in [0, 1]$ be the normalized labels such that:
$$
  \bar{L}_u = \frac{L_u - \min_{v \in V} L_v}{\max_{v \in V} L_v - \min_{v \in V} L_v},
$$
We use a color mapping function $\Psi: [0, 1] \rightarrow W_c$ to obtain natural language color labels $C_v=\Psi(\bar{L}_v)$, we use the, where $W_c \subset W$ is a subset of natural language color words like \{red, pink, yellow, ...\}. To exemplify, choosing the RGB color spectrum, we can have $\Psi = \Psi_{\text{NLC}} \circ \Psi_{\text{RGB}}$, where $\Psi_{\text{RGB}}: [0, 1] \rightarrow [0, 1]^3$ which gives RGB channels color as $\Psi_{\text{RGB}}(\bar{L}_u) \rightarrow (R_u, G_u, B_u)$, and $\Psi_{\text{NLC}}: [0, 1]^3 \rightarrow W$ which maps the RGB channels to natural language color tokens. The structure descriptor generates a textual sequence $P_{\text{cwl}}^{(i)}=\tau_{\text{cwl}}(S(V_s^{(i)}, E_s^{(i)}))$. 

Notably, we express colors using natural language descriptors rather than alphanumeric encodings such as hexadecimal color codes. Unlike numeric labels or hexadecimal color codes, natural-language color tokens correspond to semantically grounded embeddings learned during pretraining, enabling the model to exploit similarity relationships already encoded in its representation space.  

\paragraph{3. Few-shot Guidance descriptor.} A few-shot builder $\tau_e$ is utilized to sample a set of queries and answers from the task such that $I^{(i)} \subset T$. Specifically, we use a wrapper $r(.)$ to put the few-shot examples in a fixed template. Next, the builder puts all the wrapped examples with some description to generate $P_e^{(i)}=\tau_e(r(I^{(i)}))$.

\paragraph{4. Query descriptor.} The final component of the prompt is the query descriptor $\tau_{q}$. It states the problem along with the wrapped query as $P_q^{(i)}=\tau_q(r(\{Q\}))$.

Our \ourmethod translator mixes all the above components and makes the suitable structure-aware prompts as $P^{(i)}=\tau(Q^{(i)}, S(V_s^{(i)}, E_s^{(i)}))$, which could be a simple concatenation as $P^{(i)}=[P_c^{(i)}\|P_{\text{cwl}}^{(i)}\|P_e^{(i)}\|P_q^{(i)}]$. 

\paragraph{Compressed prompting for localized queries.}
The same translator supports a localized prompt construction. For a query node set $Q^{(i)}$ and a hop budget $k_n$, we define the retained node set
\[
R_{k_n}(Q^{(i)})=\{v\in V_s^{(i)}:\min_{u\in Q^{(i)}} d(v,u)\le k_n\}.
\]
When $k_n=-1$, no node filtering is applied. Otherwise, the prompt serializes only the induced subgraph on $R_{k_n}(Q^{(i)})$ and emits WL labels or color descriptors only for retained nodes. This node-level filtering is independent of the choice of textual encoding: it can be applied to the TLG-style baseline, to \lowl, to \cowl, or to \ourmethod. It therefore separates two effects that are otherwise confounded in long prompts: the amount of graph context given to the LLM and the type of structural descriptor used for the retained context.

\paragraph{Properties of the translation.}
This color-based translation is structure-, similarity-, and order-preserving, while remaining human-interpretable and aligned with LLM inductive biases. The computed labels and colors are the output of the translation function $\tau$ and are injected into the prompt. Examples of the prompt structure are in Appendix \ref{app:prompt-struct}.

\begin{table}[ht]
\centering
\scriptsize

\resizebox{\columnwidth}{!}{\begin{tabular}{lll}
\toprule
\textbf{Task} & \textbf{Time complexity} & \textbf{Objective} \\ \midrule
%\multicolumn{3}{c}{\textbf{Graph Algorithmic Tasks}} \\ \midrule
%Edge Count & \(O(1)\) (stored) & Counting \\   
%Node Count & \(O(1)\) (stored) & Counting \\   
Triangle Counting & \(O(m^{3/2})\) & Local Pattern Counting \\
%Connected Nodes & \(O(\deg(v))\) & Local Structural \\ 
%Node Degree & \(O(\deg(v))\) & Local Structural \\   
%Edge Existence & \(O(\deg(u))\) & Local Structural \\   
%Disconnected Nodes & \(O(n + \deg(v))\) & Global Structural \\ 
Cycle Check & \(O(n + m)\) & Global Structural Reasoning \\   
Reachability & \(O(n + m)\) & Global Connectivity Reasoning \\   
Shortest Path & \(O(n + m)\) & Combinatorial Reasoning \\   
Maximum Flow & \(O(n\,m^{2})\) & Optimization \\ \midrule
%\multicolumn{3}{c}{\textbf{Graph Predictive Tasks}} \\ \midrule
%Node Classification & \(O(1)\) & Predictive \\   
%Pairwise Node Classification & \(O(1)\) & Predictive \\   
%Link Prediction & \(O(\deg(u))\) & Predictive \\   
\end{tabular}}
\vspace{-1mm}
\caption{Task time complexity and characteristics (with \(n = |V|\), \(m = |E|\)) for algorithmic tasks. We also use node classification as a predictive task.
\vspace{-4mm}
\label{tab:task_complexity}}
\end{table}

\section{Experimental Results}
\label{exp:intro}
We evaluate whether ordered WL structural descriptors and their color-based translation improve LLM-based graph reasoning and prediction. %We first test core reasoning quality, then analyze why human-aligned colors help, study behavior under prompt compression, evaluate transfer to prediction tasks, and finally examine long-range, large-graph, and supporting analyses.
Code is available at: \url{https://github.com/angelozangari/CL-OWL}

\subsection{Setup}
\label{subsec:exp_setup}
The following contains major details on experimental setups. Additional details are in Appendix \ref{app:setup_details}. %for more information. 
%This section describes the experimental setup used to evaluate the augmentation methods introduced in Chapter~\ref{chap:GS}.

\textbf{Graph Tasks.} We evaluate our methods on a variety of local and global algorithmic tasks requiring different types of reasoning. 
 Table \ref{tab:task_complexity} shows the time complexity of the optimal algorithmic solver for each task. We also evaluate \textbf{node classification} as the most common graph predictive task. 

\textbf{Datasets.} For algorithmic tasks (Table \ref{tab:task_complexity}), we use synthetic graphs of multiple types, namely: Erd\H{o}s–R\'enyi, Barab\'asi–Albert, and Path. For each type, we generate graphs with sizes from \(n=5\) to \(n=100\)---beyond this, we observed that even strong LLMs (e.g., GPT-4-class) maintain zero accuracy on our hardest tasks. For node classification, we sample subgraphs from real-world datasets, specifically Cora, Citeseer, PubMed, \cite{yang2016revisiting} and OGBN-ArXiv \cite{hu2020open}. Additional details are in Appendix \ref{app:setup_details}.

%Each graph folder contains three files: (i) \texttt{data.json}: edge list, solutions for all tasks, WL labels, colors, and prompts for all variants. (ii) \texttt{.graphml} file for interoperability. (iii) \texttt{.png} visualization showing WL labels and colors by iteration.
%All generation is seeded for reproducibility.

%Other than algorithmic tasks, our experiment also extend to predictive tasks, more specifically to node classification (whose time complexity it $O(1)$). For the latter task, we sample subgraphs from real-world datasets, specifically Cora, Citeseer, PubMed, and OGBN-ArXiv.

\textbf{Baselines. } The recent work, Talk-Like-a-Graph (TLG) \cite{fatemiTalkGraphEncoding2023}, introduces various types of methods for graph-to-text translation by graph traversal. Each method employs a distinct vocabulary of words for textual node encoding, along with a special edge encoding. We compare our method against two variants. 1) \textbf{TLG-A}: node indices for node encoding and tuples representation for edges, and 2) \textbf{TLG-F}: Friends series' characters for node encoding and tuples for edge representation. For node classification, we further compare against three recent graph-language baselines: GraphText \cite{zhaoGraphTextGraphReasoning2023b}, LLaGA \cite{chenLLaGALargeLanguage2024c}, and OFA \cite{liuOneAllTraining2024b}. We use the official GraphText inference pipeline, the official released LLaGA checkpoint, and the official OFA codebase with the paper-reported supervised training setting for node classification (100 epochs, learning rate $10^{-4}$, batch size 128). \textit{We use GPT-3.5 as the base LLM for the experiments unless specified otherwise.} %\pb{Check for correctness of the encodings and change accordingly.}

\textbf{Variations of our Method. } We evaluate three variants of our method against the baselines as follows: (i) \textbf{\lowl:} Pure (Only) WL-based Labels; (ii) \textbf{\cowl:} Pure (Only) WL-based Colors; (iii) \textbf{\ourmethod:} WL-based Labels $+$ Colors.

\subsection{Graph Reasoning Tasks}
%\az{can we instead call it: Core Reasoning Performance}
\label{exp:quality}
We evaluate our methods on graph algorithmic tasks. We sample 200 graphs for each task of sizes 10-30 nodes from Barabasi-Albert (BA) and Erdos-Renyi (ER) graph types following the setup in \cite{fatemiTalkGraphEncoding2023,wangMicrostructuresAccuracyGraph2024c}. To avoid non-trivial tasks, we make graphs sufficiently connected by setting an edge probability of $p=0.2$ for ER graphs and a minimum number of four edges per node for BA graphs. These specific graph types are closer to real-world graphs compared to simpler structures, such as complete or path graphs. 
%are the scale-free network (BA) and the stochastic block model (ER), which

Results are shown in Table \ref{tab:algorithmic-tasks-gpt4o-agg}. For challenging graph tasks, our proposed WL-enriched prompting (with/without colors) achieves significant improvements over the baselines. This empirically verifies our claim, as including the WL-labels provides information from long-range dependencies that are accessible and allow reasoning on these tasks that require global structural knowledge. The coloring variants (\cowl and \ourmethod) show superior performance over the pure-WL methods. Regarding graph-level triangle counting, the WL-enriched prompts underperform the baselines. We conjecture that exact global triangle counting is closer to exhaustive local pattern matching over all triples, where extra descriptors can lengthen the prompt without directly identifying the counted motifs. In Section \ref{exp:local-motifs}, we therefore separately evaluate a localized triangle-membership task and find that WL-based cues improve local motif recognition.

\begin{table*}[ht]
\vspace{-2mm}
\centering
\small
\setlength{\tabcolsep}{5pt}
\begin{tabular}{lcccccccc}
\toprule
& \multicolumn{2}{c}{Cycle Check} & \multicolumn{2}{c}{Maximum Flow} & \multicolumn{2}{c}{Shortest Path} & \multicolumn{2}{c}{Triangle Counting} \\
\cmidrule(lr){2-3}\cmidrule(lr){4-5}\cmidrule(lr){6-7}\cmidrule(lr){8-9}
Variant & Acc.\ (\%) & MAE $\downarrow$ & Acc.\ (\%) & MAE $\downarrow$ & Acc.\ (\%) & MAE $\downarrow$ & Acc.\ (\%) & MAE $\downarrow$ \\
\midrule
\baselineOne{} & 89.50 & 0.105 & \underline{33.33} & 0.494 & 87.80 & 0.041 & \textbf{16.00} & \underline{0.456} \\
\baselineTwo{} & 91.50 & 0.085 & \textbf{36.59} & 0.445 & 83.74 & 0.050 & 11.50 & 0.523 \\
\lowl{} & 92.00 & 0.080 & \textbf{36.59} & 0.475 & \textbf{91.87} & \textbf{0.025} & 14.00 & 0.508 \\
\cowl{} & \underline{92.50} & \underline{0.075} & \textbf{36.59} & \underline{0.443} & \underline{88.62} & \underline{0.039} & 14.00 & \textbf{0.454} \\
\ourmethod{} & \textbf{93.00} & \textbf{0.070} & \textbf{36.59} & \textbf{0.438} & 86.99 & 0.042 & \underline{14.50} & 0.469 \\
\bottomrule
\end{tabular}
\caption{Results on graph algorithmic tasks averaged over BA and ER datasets using \texttt{gpt-4o} (additional results are in Appendix \ref{app:additional_results}). \ourmethod generally achieves superior performance over the baselines. Overall, the color-based variants (\ourmethod and \cowl) demonstrate more improvements compared to others.}
\label{tab:algorithmic-tasks-gpt4o-agg}
\vspace{-4mm}
\end{table*}

\subsection{Structural Information Compression}
\label{exp:compression-scalability}
The ordered WL labels and their color-based translation expose graph-level structural information in a compact node-level form. %To verify this claim we perform  prompt compression.
%A central claim of our method is that ordered WL labels and their color-based translation expose graph-level structural information in a compact node-level form. This means the descriptors should remain useful even when the serialized graph context is aggressively truncated, and their advantage should become more apparent as graph size increases. We test this claim from two complementary angles: prompt compression and large-graph scalability.
To verify this claim we first evaluate the compressed prompting variant introduced in Section~\ref{sec:method_translation}. The idea is to retain only the induced subgraph within a fixed hop budget around the query node while preserving WL-derived descriptors for the retained nodes. If the descriptors indeed summarize non-local structural information, then they can compensate for the loss of explicit graph context. We test this on Cora for node classification task using 50-node subgraphs. For the TLG-A baseline, we vary the retained neighborhood from 1-hop to 3-hop around the target node; for WL-based variants, we use only the 1-hop neighborhood. Table~\ref{tab:compressed-nodecls} shows that the baseline degrades sharply under this restriction (accuracy of 30-45\%). In contrast, \cowl{} reaches 75.0\% accuracy with only 1-hop context, and \ourmethod{} reaches 70.0\% while achieving the best MAE and RMSE. Thus, the WL-based variants outperform the baseline while observing strictly less serialized graph context, indicating that the descriptors encode information that the baseline must otherwise recover from a larger neighborhood.

\begin{table}[t]
\centering
\small
\setlength{\tabcolsep}{5pt}
\begin{tabular}{lccc}
\toprule
Variant & Acc.\ (\%) & MAE $\downarrow$ & RMSE $\downarrow$ \\
\midrule
\baselineOne{} (1-hop) & 30.0 & 0.700 & 0.837 \\
\baselineOne{} (2-hop) & 40.0 & 0.600 & 0.775 \\
\baselineOne{} (3-hop) & 45.0 & 0.550 & 0.742 \\
\cowl{} (1-hop) & \textbf{75.0} & 0.211 & 0.459 \\
\lowl{} (1-hop) & 50.0 & 0.500 & 0.707 \\
\ourmethod{} (1-hop) & \underline{70.0} & \textbf{0.176} & \textbf{0.420} \\
\bottomrule
\end{tabular}
\caption{Compressed prompting on Cora node classification. The hop count specifies the retained neighborhood around the target node before serialization.}
\vspace{-3mm}
\label{tab:compressed-nodecls}
\end{table}

We observe similar phenomenon in prompt length. Table~\ref{tab:prompt-length-scalability} reports the mean number of characters for compressed node-classification prompts on Cora across graph sizes from 50 to 1000 nodes and for two graph types. On BA graphs, the baseline prompt length at $n=1000$ is about $4\times$ that of the WL-based variants. 
% decreases from 1705906 characters for the baseline to 362409--368649 characters for the WL-based variants, corresponding to about $4\times$ reduction . 
On ER graphs, the reduction is even larger (nearly $15\times$). The effect is strongest on sparse graphs, where a 1-hop retained subgraph remains small while the full serialization continues to scale with graph size. This shows that the benefit of WL-based prompting is not merely predictive but also representational.%: the descriptors preserve useful structural information while greatly reducing the amount of text that must be passed to the LLM.

\begin{table}[t]
\centering
\small
\setlength{\tabcolsep}{5pt}
\begin{tabular}{llcccc}
\toprule
Graph & Method & $50$ & $100$ & $500$ & $1000$ \\
\midrule
\multirow{4}{*}{BA}
& \baselineOne{}   & 4442 & 15122 & 390949 & 1705906 \\
& \cowl{}          & \underline{2595} & \underline{6148} & \underline{188525} & \underline{364108} \\
& \lowl{}          & \textbf{2518} & \textbf{5943} & \textbf{187435} & \textbf{362409} \\
& \ourmethod{}     & 3172 & 6957 & 191561 & 368648 \\
\midrule
\multirow{4}{*}{ER}
& \baselineOne{}   & \textbf{1382} & 1842 & 9045 & 29020 \\
& \cowl{}          & \underline{1393} & \textbf{1399} & \textbf{1454} & \underline{1530} \\
& \lowl{}          & 1409 & \underline{1414} & \underline{1458} & \textbf{1521} \\
& \ourmethod{}     & 1813 & 1822 & 1898 & 2002 \\
\bottomrule
\end{tabular}
\caption{Mean prompt length (characters) for compressed node-classification prompts across graph types and sizes. The TLG-A baseline serializes the full retained graph context, while WL-based variants use 1-hop node filtering around the query node.}
\vspace{-3mm}
\label{tab:prompt-length-scalability}
\end{table}

\begin{comment}
\begin{table}[t]
\centering
\small
\setlength{\tabcolsep}{4pt}
\begin{tabular}{llcccc}
\toprule
Graph & Method & $50$ & $100$ & $500$ & $1000$ \\
\midrule
\multirow{4}{*}{BA}
& \baselineOne{} & 4.4K & 15.1K & 390.9K & 1.71M \\
& \cowl{}        & \underline{2.6K} & \underline{6.1K} & \underline{188.5K} & \underline{364.1K} \\
& \lowl{}        & \textbf{2.5K} & \textbf{5.9K} & \textbf{187.4K} & \textbf{362.4K} \\
& \ourmethod{}   & 3.2K & 7.0K & 191.6K & 368.6K \\
\midrule
\multirow{4}{*}{ER}
& \baselineOne{} & \textbf{1.4K} & 1.8K & 9.0K & 29.0K \\
& \cowl{}        & \underline{1.4K} & \textbf{1.4K} & \textbf{1.5K} & \underline{1.5K} \\
& \lowl{}        & 1.4K & \underline{1.4K} & \underline{1.5K} & \textbf{1.5K} \\
& \ourmethod{}   & 1.8K & 1.8K & 1.9K & 2.0K \\
\bottomrule
\end{tabular}
\caption{Mean prompt length for compressed node-classification prompts. Values are shown in thousands (K) or millions (M) of characters. Lower is better.}
\label{tab:prompt-length-scalability}
\end{table}
\end{comment}
%\az{reframe conclusion} Taken together, these compression and scalability results support the same underlying conclusion. Ordered WL labels and their color-based translation are informative because they compress graph-level structure into node-level descriptors that remain useful when explicit context is reduced and when graph size increases. Their benefit is therefore not only improved accuracy on a fixed benchmark, but a more information-efficient graph-to-text representation for LLMs.

\subsection{Importance of Colors}
%\az{or could be: Why Human-Aligned Colors Help}
\label{exp:importance-of-colors}

Our method translates ordered WL labels into natural-language color words rather than into arbitrary identifiers. To test whether the observed gains come from the structural information or from expressing that information in a semantically grounded vocabulary due to compatibility with LLM's linguistic priors.

First, we compare color words against alternative vocabularies that preserve the same WL equivalence classes but remove their natural-language grounding. We evaluate node classification on Cora using over-sampled subgraphs of sizes 10, 20, 30, 40, and 50, and replace color words with either synthetic hue identifiers (e.g., ``Hue $i$'') or human names. Table~\ref{tab:wl-encoding-comparison} shows that natural-language colors substantially outperform both alternatives. 

Since these variants preserve the same underlying WL partition, the large performance gap cannot be explained solely by structural information, meaning the LLM better exploits WL-derived structure when it is expressed using semantically meaningful color terms rather than opaque tokens.

\begin{table}[t]
\centering
\vspace{-2mm}
\small
\setlength{\tabcolsep}{5pt}
\begin{tabular}{lccc}
\toprule
Variant & Accuracy (\%) & MAE $\downarrow$ & RMSE $\downarrow$ \\
\midrule
\baselineOne{} & 49.0 & 0.510 & 0.714 \\
\cowl{} & \textbf{68.9} & \textbf{0.311} & \textbf{0.558} \\
\lowl{} & \underline{63.6} & 0.364 & 0.604 \\
\ourmethod{} & 66.4 & \underline{0.336} & \underline{0.580} \\
WL Colors (hue) & 15.2 & 0.848 & 0.921 \\
WL Colors (names) & 6.3 & 0.937 & 0.968 \\
\bottomrule
\end{tabular}
\caption{Effect of the textual realization on node classification. Hue identifiers and names preserve the same WL class assignment as \cowl{} but replace semantically grounded color words with opaque tokens.}
\vspace{-3mm}
\label{tab:wl-encoding-comparison}
\end{table}

We next study how many distinct color descriptors are needed, since large graphs may yield more WL labels that cannot be uniquely represented with a reasonably small palette. We vary the number of available color descriptors in \cowl{} on 50-node Cora subgraphs. We use three palettes with only base colors (C3, C6, C9-0), and two larger palettes formed by combining nine base colors with lightness modifiers, yielding 27 colors (C9-3) and 45 colors (C9-5). In this dataset, WL refinement produces about $32$ distinct labels per graph, so limited palettes induce collisions, where multiple WL labels are mapped to the same color descriptor.

Table~\ref{tab:color-cardinality-ablation} shows two main patterns. First, all color-based variants outperform the baseline, even when the collision factor is high, showing the usefulness of approximate grouping of structurally similar nodes. Second, performance improves with more color descriptors as collisions decrease. The strongest results are obtained for collision factors between roughly 1 and 5, while performance drops again when the palette becomes overly fine-grained. This suggests that 
% the benefit of color descriptors comes from a balance between discrimination and semantic compactness: 
too few colors merge distinct structural roles, while too many descriptors weaken the regularizing effect of shared human-interpretable categories.

\begin{table}[t]

\centering
\small
\setlength{\tabcolsep}{4pt}
\begin{tabular}{lcccc}
\toprule
Variant & Acc.\ (\%) & MAE $\downarrow$ & \# Colors & CF \\
\midrule
\baselineOne{} & 60.0 & 0.400 &  -- & -- \\
C3 & 65.0 & 0.350 & 3 & 10.6 \\
C6 & \textbf{75.0} & \textbf{0.250} & 6 & 5.3 \\
C9-0 & \textbf{75.0} & \textbf{0.250} & 9 & 3.5 \\
C9-3 & \textbf{75.0} & \textbf{0.250} & 27 & 1.2 \\
C9-5 & \underline{70.0} & \underline{0.300} & 45 & 0.7 \\
\bottomrule
\end{tabular}
\caption{Color-cardinality ablation. C3, C6, and C9-0 use 3, 6, and 9 base color words, respectively. C9-3 and C9-5 augment the 9-color palette with 3 and 5 lightness levels (i.e., 27 and 45 distinct color descriptors). Collision denotes the ratio between the number of WL labels and available color descriptors. CF is collision factor.}
\vspace{-4mm}
\label{tab:color-cardinality-ablation}
\end{table}

\subsection{Performance on Prediction Tasks}
\label{exp:predictive}
Besides the algorithmic tasks, we also show the effectiveness of \ourmethod for node classification. Table \ref{tab:f1-node_classification} presents the results. We first sample 30 graphs of five different sizes (10-50 nodes). Then, we sample 150 subgraphs from all the graphs. Because of this, the dataset contains 600 graph-task instances in total. Overall, incorporating WL-based structural information improves performance over the baseline, indicating that explicit structural cues can benefit predictive tasks beyond purely algorithmic reasoning. The color-only variant (\cowl) achieves the strongest results on Cora, Citeseer, and PubMed. In contrast, the color-label variant (\ourmethod) performs best on the larger and structurally more complex OGBN-ArXiv dataset, where human-tangible similarity cues appear to better support long-range structural dependencies. These results highlight that the relative benefit of different encodings depends on dataset scale and complexity, with color-based representations becoming more advantageous as graphs grow larger.

\begin{table}[t]
\centering
\vspace{-2mm}
\small
\renewcommand{\arraystretch}{1.15}
\setlength{\tabcolsep}{4pt}
\begin{tabular}{lcccc}
\toprule
Method & Cora & Citeseer & PubMed & OGBN-ArXiv \\
\midrule
\baselineOne{} & 11.35 & 9.01 & 5.70 & 22.76 \\
\lowl& 18.78 & \underline{13.83} & \underline{13.23} & 22.61 \\
\cowl & \textbf{22.30} & \textbf{20.08} & \textbf{14.33} & \underline{26.87} \\
\ourmethod& \underline{20.93} & 13.63 & 11.94 & \textbf{40.94} \\
\bottomrule
\end{tabular}
\caption{Node-classification performance measured by F1-Macro (\%). The color-only variant performs best on Cora, Citeseer and PubMed, with the Colors and Full method trailing behind. On ArXiv the full method performs best, followed by the labels.
%Best results per dataset are in bold; second-best are underlined.
}
\vspace{-2mm}
\label{tab:f1-node_classification}
\end{table}

\textit{Additional experiments including the effects on long range dependencies, and scalability are available in Appendix \ref{app:additional_results}}.

%Additional experiments including the effects on long range dependencies, and scalability, comparison with additional baselines, local motifs, dataset- labels- and colors- metrics and the full results for the core reasoning tasks are available in Appendix \ref{app:additional_results}

%\clearpage
\section{Conclusion}
We have studied how LLMs can be better adapted to graph-structured reasoning through human-tangible auxiliary structural information. By introducing \ourmethod, an ordered WL–based graph-to-text representation and mapping structural equivalence classes to natural-language color tokens, we provide a principled way to expose graph structure while preserving permutation invariance and linguistic alignment. We have provided a theoretical analysis connecting ordered WL refinement to distance-weighted connectivity, and a systematic empirical evaluation across diverse graph tasks and graph families. Our results show that human-tangible structural cues can substantially improve LLM performance on graph reasoning, particularly for tasks requiring global structural understanding.

%\paragraph{Limitations:} although - while - limitations of this work are

%\textbf{Augmentation Strategy.} Our augmentation strategy therefore uses WL refinement as the underlying structural descriptor and maps its outputs into human-understandable colors. This produces a compact textual representation of node-level structural roles that can be inserted directly into the prompt. This approach is both task-agnostic and graph-agnostic: it applies equally to synthetic graphs, real-world graphs, and tasks ranging from simple structural queries to complex reasoning or combinatorial problems.  
\clearpage
%\clearpage
\section{Ethical Considerations}
%In this work, we have built methods for generating explanations for graph neural network predictions. We do not foresee any ethical issues from our study. 
This work focuses on improving the representation of graph-structured data for large language models (LLMs) through human-interpretable auxiliary structural information. Our method operates solely on abstract graph topology and synthetic or standard benchmark datasets, and does not involve human subjects, personal data, or sensitive attributes. We do not introduce any deployment mechanisms that could amplify social bias, privacy risks, or misuse beyond those already associated with existing LLMs. As such, we do not foresee any ethical issues from our study.

\textbf{Potential Risks.} While our method operates on abstract graph structures and does not involve sensitive data, it introduces a representation bias by encoding structural information through human-interpretable tokens (e.g., colors). Such encodings may implicitly prioritize certain structural patterns over others, potentially leading to misinterpretation or over-reliance by LLMs in downstream tasks. Additionally, as our approach is prompt-based and leverages pretrained LLMs, it inherits known limitations of these models, including sensitivity to prompt design, lack of robustness to distribution shifts, and potential hallucinations in reasoning.

\section{Limitations}

We present some limitations of our work:

\begin{itemize}
\item \textit{Generalization to NP-hard Problems.} Although we evaluate a diverse set of synthetic and real-world graphs, the task set does not exhaustively cover all graph reasoning problems, especially the NP-hard ones. 
\item \textit{Human-interpretable Encodings are Heuristic.}
 While color-based encoding aligns well with human and linguistic priors, it is ultimately a heuristic design choice. Alternative interpretable representations may yield different or superior performance.
\end{itemize}
%\item Our method projects GNN embeddings into the LLM embedding space and interleaves them with text to form hybrid prompts. However, batching multiple nodes leads to degraded performance and hallucinations, likely due to interference between soft prompts.
    %\item Hallucinations when projecting GNN embeddings into the LLM space. Our method projects GNN embeddings into the LLM embedding space and interleaves them with text to form hybrid prompts. However, batching multiple nodes leads to degraded performance and hallucinations, likely due to interference between soft prompts.
  %  \item While \methodname achieves high empirical fidelity, it does not provide formal guarantees of faithfulness to the GNN’s internal decision logic, as the explanation relies on the LLM’s interpretation of projected embeddings rather than directly tracing GNN computation paths.
 %   \item Although our method incorporated a projector to align the embeddings of the GNN with the LLM space of token embeddings, the objectives we introduce for training the projector are not directly tailored for the task of explanation. More optimal objectives that adhere to the explanation task could bring substantial benefits.
% \end{itemize}
%\section*{Discussion}

%this is a doubt
\clearpage
\bibliography{custom}

\appendix
\clearpage

\section{Additional Details on Related Work}
\label{app:rel-work}

%A central challenge in applying large language models (LLMs) to graph problems is that graphs are non-sequential objects, whereas LLMs operate exclusively on text. Consequently, prior work primarily differs in how graph structure is rendered into textual form. Below, we reorganize existing graph--LLM methods according to the graph-to-text interface they employ, building on the broader literature surveyed in Section~\ref{sec:rw_graph_problems}.

\paragraph{Locality-driven interfaces.}
The most common strategy restricts the textual description of a graph to local neighborhoods, such as $k$-hop ego-nets, short paths, or small induced subgraphs. This approach reduces prompt length and stabilizes generation, but obscures long-range dependencies required for global reasoning. Many early and recent solver-oriented works adopt locality-driven serialization, including NLGraph-style benchmarks and graph-in-text prompting methods \cite{wangCanLanguageModels2024,guoGPT4GraphCanLarge2023,sunLargeLanguageModels2025,zhaoGraphTextGraphReasoning2023b}. Similar locality biases appear in text-attributed graph pipelines and hybrid graph--LLM systems that operate on bounded-depth neighborhoods \cite{qinDisentangledRepresentationLearning2024,yeLanguageAllGraph2024,chenLLaGALargeLanguage2024c}. Empirical analyses show that such locality constraints correlate with failures on tasks requiring global connectivity or cycle reasoning \cite{wangMicrostructuresAccuracyGraph2024c,fuBringComplexGeometric2025}.

\paragraph{Similarity-driven interfaces.}
A smaller body of work selects which nodes or substructures to expose to the LLM based on similarity under a learned encoder, often using cosine similarity in embedding space. The goal is to group semantically related nodes rather than strictly adjacent ones. However, similarity in text or embedding space does not necessarily correspond to topological similarity, and these methods introduce additional design complexity. Representative examples include similarity-based neighbor selection and retrieval strategies for graph LLMs \cite{liSimilaritybasedNeighborSelection2024,xuHowMakeLLMs2025,sunGraphICLUnlockingGraph2025,liAreLargeLanguage2025}.

\paragraph{Embedding-based interfaces.}
Another line of work injects continuous graph representations into LLMs by first encoding structure with a GNN and then aligning the resulting embeddings with the LLM input space via projectors or soft prompts. These methods can convey rich structural information but reduce transparency and often require training or tuning. Examples include LLaGA and related graph--LLM architectures that integrate GNN embeddings into language models \cite{chenLLaGALargeLanguage2024c,liuCanWeSoft2024c,zhuParameterEfficientTuningLarge2024,heUniGraphLearningUnified2025,tangGraphGPTGraphInstruction2024}.

\paragraph{Textual mapping to discrete symbols.}
A complementary strategy converts structural or embedding-derived information into discrete textual tokens that can be directly consumed by LLMs. Talk Like a Graph systematically studies how different graph encodings affect LLM reasoning \cite{fatemiTalkGraphEncoding2023}, while other works map graph structure or learned signals into labels or symbolic descriptions \cite{zhaoGraphTextGraphReasoning2023b,tangGraphGPTGraphInstruction2024}. AuGLM is among the few approaches that explicitly translate embedding-derived information into textual labels for LLM consumption \cite{xuHowMakeLLMs2025}. These methods are closest in spirit to our work.

 Our approach belongs to the textual mapping category. Unlike prior methods that rely on arbitrary numeric labels or post hoc embedding translation, we derive a task-agnostic structural signal using Weisfeiler-Leman refinement and encode it using human-interpretable, similarity-preserving textual attributes. This design maintains a fully text-based interface while exposing non-local structural information aligned with the inductive biases of LLMs.

\section{Additional Details on Our Method}
\subsection{Proof of Theorem \ref{thm:wl-centrality}}
\label{app:proof}
Here we provide the proof for  Theorem \ref{thm:wl-centrality} described in Section \ref{sec:method_structural_augmentation}.
\begin{proof}[Proof sketch]
\textbf{(1)} At $t=0$ all nodes have $\ell_v^{(0)}=1$, hence
$m_v^{(0)}=(1,(1,1,\dots,1))$ where the tuple length equals $\deg(v)$.
Lexicographic order compares tuples first by prefix and then by length; since all entries are $1$,
larger degree yields a lexicographically larger message, so $\mathrm{ID}^{(0)}$ assigns a larger integer:
$\deg(v)>\deg(w)\Rightarrow \ell_v^{(1)}>\ell_w^{(1)}$.

\smallskip
\textbf{(2)} Under the tree-unfolding assumption up to depth $T$, the $T$-hop rooted neighborhood of a node
can be represented by a rooted tree in which the number of nodes at depth $k$ equals $|S_k(\cdot)|$.
Ordered 1-WL on such a tree corresponds to bottom-up aggregation of sorted multisets of child labels.
If $|S_k(v)|\ge |S_k(w)|$ for all $k\le T$ with strict inequality at some $k^\star$, then at the earliest
depth where the rooted trees differ, the multiset (hence the sorted tuple) of child labels for $v$
strictly dominates that of $w$ in lexicographic order. Because $\mathrm{ID}^{(t)}$ is order-preserving
over messages, this dominance propagates through subsequent iterations, yielding
$\ell_v^{(T)}>\ell_w^{(T)}$.

\smallskip
\textbf{(3)} If $\alpha_k$ are nonincreasing and positive, shell dominance with a strict inequality implies
$C_T(v)>C_T(w)$ immediately by summation. Combining with (2) yields consistency between the induced WL order
and $C_T(\cdot)$ for this class of local neighborhoods.
\end{proof}

\subsection{The Pseudocodes}
The pseudocode of ordered 1-WL refinement is shown in Algorithm \ref{alg:wl_simple}. The same for color mapping is shown in Algorithm \ref{alg:color_simple}.

\begin{algorithm}[t]
\caption{Ordered 1-WL Refinement (Canonical Relabeling)}
\label{alg:wl_simple}
\begin{algorithmic}[1]
\Require Graph $G=(V,E)$
\Ensure Final node labels $\ell:V\to\mathbb{N}$ (and optionally the full history)

\State Initialize $\ell_v \gets 1$ for all $v\in V$
\Repeat
    \ForAll{$v\in V$}
        \State $M_v \gets \mathrm{sort}\big(\{\!\{\ell_u : u\in \mathcal{N}(v)\}\!\}\big)$
        \State $m_v \gets (\ell_v,\, M_v)$ \Comment{node message}
    \EndFor
    \State $U \gets \mathrm{sort}\big(\{ m_v : v\in V\}\big)$ \Comment{unique messages, lexicographically sorted}
    \State Define $\mathrm{ID}(m)$ as the index of $m$ in $U$ (0-based or 1-based)
    \ForAll{$v\in V$}
        \State $\ell'_v \gets \mathrm{ID}(m_v)$
    \EndFor
    \State $\ell \gets \ell'$
\Until{$\ell$ does not change}
\State \Return $\ell$
\end{algorithmic}
\end{algorithm}

\begin{algorithm}[t]
\caption{Map WL Labels to Color Names}
\label{alg:color_simple}
\begin{algorithmic}[1]
\Require Labels $\ell_v$ for all $v\in V$
\Ensure Color annotation $\mathrm{Color}(v)$ for all $v\in V$

\State $m_{\min} \gets \min_{v\in V}\ell_v$, \quad $m_{\max} \gets \max_{v\in V}\ell_v$
\State $r \gets \max(1,\, m_{\max}-m_{\min})$
\ForAll{$v\in V$}
    \State $x_v \gets (\ell_v - m_{\min})/r$ \Comment{$x_v\in[0,1]$}
    \State $h_v \gets h_{\min} + (h_{\max}-h_{\min})\cdot x_v$ \Comment{e.g., $[h_{\min},h_{\max}]=[0^\circ,180^\circ]$}
    \State $\mathrm{Color}(v) \gets \mathrm{HueToName}(h_v)$
\EndFor
\State \Return $\mathrm{Color}$
\end{algorithmic}
\end{algorithm}

\subsection{Prompt Structure}
\label{app:prompt-struct}

\begin{FullWidthPromptFigure}{Prompt template for max-flow using the \ourmethod{} method}{fig:prompt-maxflow}

\PromptSection{SysBlue}{SYSTEM}{%
You are an expert in graph theory and graph machine learning. You have deep knowledge of the Weisfeiler-Leman (WL) algorithm, structural graph analysis, and graph visualization techniques.
Graph structure is provided as adjacency lists: each line is `node: neighbor1, neighbor2, ...` and isolated nodes use `(none)`.
Always follow the example format and present the final answer after the `<<ANSWER>>` marker.
Reason carefully about the graph before responding.
}

\PromptSection{UserGreen}{USER (task and format)}{%
Your task is to return the maximum flow value between the <<TARGET\_PAIR>> with unit capacities. The following example demonstrates the format.
A graph is represented as adjacency lists where each line is `node: neighbor1, neighbor2, ...` (use `(none)` for isolated nodes).
`<<WL\_LABELS>>` shows Weisfeiler--Leman structural labels for each node. WL labels capture the structural role of each node based on its neighborhood. WL labels are given the template: WL(i)=k, meaning that node i has WL label k.
Nodes with closer labels are more similar; for instance if WL(node\_1)=1, WL(node\_2)=4, WL(node\_3)=9, then node\_1 is more similar to node\_2 than to node\_3.
`<<COLORS>>` shows colors assigned based on WL label similarity. Colors capture the structural role of each node based on its neighborhood. Colors are given the template: Color(i)=c, meaning that node i has color c.
Colors can be used as a similarity metric among nodes (nodes with similar colors are more similar).
`<<TARGET\_NODE>>` marks the queried node; `<<TARGET\_PAIR>> a,b` marks the queried pair for edge/pairwise tasks.

Always use the provided markers consistently and place the final answer after the `<<ANSWER>>` marker. Think step-by-step about the graph before responding.
}

\PromptSection{GraphPurple}{<<GRAPH>>}{%
<<GRAPH>>
0: 1, 2, 3, 4, 5, 6, 8.
1: 0, 5, 6, 7, 8, 9.
2: 0, 5.
3: 0, 5, 7, 8, 9.
4: 0, 6, 7.
5: 0, 1, 2, 3, 6, 7, 8.
6: 0, 1, 4, 5, 9.
7: 1, 3, 4, 5.
8: 0, 1, 3, 5, 9.
9: 1, 3, 6, 8.
}

\PromptSection{WlOrange}{<<WL\_LABELS>>}{%
WL(0):8, WL(1):7, WL(2):0, WL(3):5, WL(4):1,
WL(5):9, WL(6):4, WL(7):2, WL(8):6, WL(9):3
}

\PromptSection{ColorTeal}{<<COLORS>>}{%
Color(0):teal, Color(1):lime, Color(2):red, Color(3):green, Color(4):crimson,
Color(5):teal, Color(6):lime, Color(7):orange, Color(8):green, Color(9):orange
}

\PromptSection{TargetRed}{<<TARGET\_PAIR>>}{%
<<TARGET\_PAIR>> 2,9

Use the provided markers consistently and place the final answer after the <<ANSWER>> marker. Think step-by-step about the graph before responding.
}

\end{FullWidthPromptFigure}

\begin{FullWidthPromptFigure}{Prompt template for cycle check using \lowl{}}{fig:prompt-cycle-wl-labels}

\PromptSection{SysBlue}{SYSTEM}{%
You are an expert in graph theory and graph machine learning. You have deep knowledge of the Weisfeiler-Leman (WL) algorithm and structural graph analysis.

Graph structure is provided as adjacency lists: each line is `node: neighbor1, neighbor2, ...` and isolated nodes use `(none)`.

Always follow the example format and present the final answer after the `<<ANSWER>>` marker.

Reason carefully about the graph before responding.
}

\PromptSection{UserGreen}{USER (task and format)}{%
Your task is to answer Yes/No whether the <<TARGET\_NODE>> is part of any cycle. The following example demonstrates the format.
A graph is represented as adjacency lists where each line is `node: neighbor1, neighbor2, ...` (use `(none)` for isolated nodes).
`<<WL\_LABELS>>` shows Weisfeiler--Leman structural labels for each node. WL labels capture the structural role of each node based on its neighborhood. WL labels are given the template: WL(i)=k, meaning that node i has WL label k.
Nodes with closer labels are more similar; for instance if WL(node\_1)=1, WL(node\_2)=4, WL(node\_3)=9, then node\_1 is more similar to node\_2 than to node\_3.
`<<TARGET\_NODE>>` marks the queried node; `<<TARGET\_PAIR>> a,b` marks the queried pair for edge/pairwise tasks.

Always use the provided markers consistently and place the final answer after the `<<ANSWER>>` marker. Think step-by-step about the graph before responding.
}

\PromptSection{GraphPurple}{<<GRAPH>>}{%
<<GRAPH>>
0: 1, 2, 3, 4, 5, 6, 7, 8, 9
1: 0, 7, 8, 9
2: 0, 7
3: 0, 8, 9
4: 0, 7
5: 0, 7, 8, 9
6: 0, 7, 8
7: 0, 1, 2, 4, 5, 6, 8, 9
8: 0, 1, 3, 5, 6, 7, 9
9: 0, 1, 3, 5, 7, 8
}

\PromptSection{WlOrange}{<<WL\_LABELS>>}{%
<<WL\_LABELS>>
WL(0):7, WL(1):3, WL(2):0, WL(3):1, WL(4):0,
WL(5):3, WL(6):2, WL(7):6, WL(8):5, WL(9):4
}

\PromptSection{TargetRed}{<<TARGET\_NODE>>}{%
<<TARGET\_NODE>> 7

Use the provided markers consistently and place the final answer after the `<<ANSWER>>` marker. Think step-by-step about the graph before responding.
}

\end{FullWidthPromptFigure}

\begin{FullWidthPromptFigure}{Prompt template for shortest path using \cowl{}}{fig:prompt-shortestpath-wl-colors}

\PromptSection{SysBlue}{SYSTEM}{%
You are an expert in graph theory and graph machine learning. You have deep knowledge of the Weisfeiler-Leman (WL) algorithm, structural graph analysis, and graph visualization techniques.

Graph structure is provided as adjacency lists: each line is `node: neighbor1, neighbor2, ...` and isolated nodes use `(none)`.

Always follow the example format and present the final answer after the `<<ANSWER>>` marker.

Reason carefully about the graph before responding.
}

\PromptSection{UserGreen}{USER (task and format)}{%
Your task is to return the shortest path length between the <<TARGET\_PAIR>> (inf if no path). The following example demonstrates the format.
A graph is represented as adjacency lists where each line is `node: neighbor1, neighbor2, ...` (use `(none)` for isolated nodes).
`<<COLORS>>` shows colors assigned based on WL label similarity. Colors capture the structural role of each node based on its neighborhood. Colors are given the template: Color(i)=c, meaning that node i has color c.
Colors can be used as a similarity metric among nodes (nodes with similar colors are more similar).
`<<TARGET\_NODE>>` marks the queried node; `<<TARGET\_PAIR>> a,b` marks the queried pair for edge/pairwise tasks.

Always use the provided markers consistently and place the final answer after the `<<ANSWER>>` marker. Think step-by-step about the graph before responding.
}

\PromptSection{GraphPurple}{<<GRAPH>>}{%
<<GRAPH>>
0: 1, 2, 3, 4, 5, 6, 8
1: 0, 5, 6, 7, 8, 9
2: 0, 5
3: 0, 5, 7, 8, 9
4: 0, 6, 7
5: 0, 1, 2, 3, 6, 7, 8
6: 0, 1, 4, 5, 9
7: 1, 3, 4, 5
8: 0, 1, 3, 5, 9
9: 1, 3, 6, 8
}

\PromptSection{ColorTeal}{<<COLORS>>}{%
<<COLORS>>
Color(0):teal, Color(1):lime, Color(2):red, Color(3):green, Color(4):crimson,
Color(5):teal, Color(6):lime, Color(7):orange, Color(8):green, Color(9):orange
}

\PromptSection{TargetRed}{<<TARGET\_PAIR>>}{%
<<TARGET\_PAIR>> 2,9

Use the provided markers consistently and place the final answer after the `<<ANSWER>>` marker. Think step-by-step about the graph before responding.
}

\end{FullWidthPromptFigure}

The prompt follows a clearly defined structure. First of all it is split in two parts: system and user prompt.
To better visualize the prompt structure we can refer to figures \ref{fig:prompt-maxflow}, \ref{fig:prompt-cycle-wl-labels}, \ref{fig:prompt-shortestpath-wl-colors}. More specifically, Figure \ref{fig:prompt-maxflow} solves max-flow with the \ourmethod{} method, Figure \ref{fig:prompt-cycle-wl-labels} solves cycle check with the \lowl{} method and Figure \ref{fig:prompt-shortestpath-wl-colors} solves shortest path with the \cowl{} method.
The system prompt contains information that is useful to explain to the LLM its role in the conversation. It contains information such as "you are an expert in ...", explanation of the syntax that will be used and general guidelines (e.g. "think carefully before responding").
The user prompt contains the task definition, the graph and the explanation for both labels and colors.
Blocks are marked explicitly using tokens such as \texttt{<<GRAPH>>}, \texttt{<<WL\_LABELS>>}, or \texttt{<<COLORS>>} to simplify parsing. Answer template: forced answer pattern such as \texttt{<<ANSWER>> Yes/No} or \texttt{<<ANSWER>> [number]} to try to make evaluation deterministic. We additionally include a short instruction such as “Think step-by-step and reason to solve the task.” This prompt pattern is widely used in LLM literature and helps stabilize reasoning without introducing model-specific biases.

\section{Reproducibility}
\textbf{Code.} The code is available at: \url{https://github.com/angelozangari/CL-OWL}

The entire pipeline is designed to be easily executable and reproducible. Easiness of \textbf{executability} is guaranteed by one main \textit{./run.sh} script that automates all the pipeline stages. This includes the initial setup, which ensures that all the required packages are installed, and sets up the python virtual environment with the required packages. Then one can proceed with the dataset generation, followed by the LLM testing, the automatic parsing and the computation of the results. Each of these steps is executable singularly with a dedicated script command.
Each configuration was evaluated with a single run using a different random seed. We do not report standard deviation as it is not highly relevant in our evaluation context. We generate different random graphs and the answers of the graph queries are averaged on these graphs.

Regarding \textbf{reproducibility}, all the parameters in the pipeline are in a dedicated config/ folder, which contains exclusively \textit{.yaml} files, in which all the needed parameters for dataset generation, LLM testing and reporting reside. Moreover, all the parts of the prompts (from the pre-task definition to the task instruction) have their content entirely confined to the \textit{config/} directory too. This way, if anyone wants to change the content of the prompt they can simply edit the configuration file of the prompt and avoid modifying the code.

\section{Detailed Experimental Setup}
\label{app:setup_details}

%BA graphs naturally produce imbalanced binary tasks due to high connectivity

% Graph-level Statistics
\begin{table}[t]
\centering
\resizebox{\columnwidth}{!}{\begin{tabular}{llcccc}
\toprule
Graph Type & Size & Avg Edges & Avg Degree & Avg Diameter & Density \\
\midrule
BA & 10 & $20.4 \pm 5.6$ & $4.08 \pm 1.12$ & $2.1 \pm 0.3$ & 0.454 \\
BA & 15 & $43.7 \pm 12.2$ & $5.83 \pm 1.63$ & $2.4 \pm 0.5$ & 0.417 \\
BA & 20 & $76.4 \pm 17.1$ & $7.64 \pm 1.71$ & $2.4 \pm 0.5$ & 0.402 \\
BA & 25 & $114.2 \pm 31.6$ & $9.13 \pm 2.53$ & $2.4 \pm 0.5$ & 0.381 \\
BA & 30 & $171.3 \pm 40.6$ & $11.42 \pm 2.71$ & $2.3 \pm 0.5$ & 0.394 \\
ER & 10 & $9.7 \pm 2.6$ & $1.95 \pm 0.51$ & $4.8 \pm 0.9$ & 0.216 \\
ER & 15 & $20.1 \pm 3.6$ & $2.68 \pm 0.48$ & $5.3 \pm 1.5$ & 0.192 \\
ER & 20 & $38.2 \pm 5.8$ & $3.82 \pm 0.58$ & $4.5 \pm 0.8$ & 0.201 \\
ER & 25 & $59.9 \pm 7.2$ & $4.79 \pm 0.58$ & $4.0 \pm 0.3$ & 0.200 \\
ER & 30 & $87.0 \pm 8.8$ & $5.80 \pm 0.59$ & $4.0 \pm 0.7$ & 0.200 \\
\bottomrule
\end{tabular}}
\caption{Synthetic graph properties by type and size for the dataset used in algorithmic tasks \ref{exp:quality}. BA denotes Barabasi Albert and ER denotes Erdos Renyi.}
\label{tab:graph-stats-reasoning}
\end{table}

% long range deps
% Graph-level Statistics
\begin{table}[t]
\centering
\resizebox{\columnwidth}{!}{\begin{tabular}{llcccc}
\toprule
Graph Type & Size & Avg Edges & Avg Degree & Avg Diameter & Density \\
\midrule
Erdos Renyi & 75 & $41.0 \pm 0.0$ & $1.09 \pm 0.00$ & N/A & 0.015 \\
Erdos Renyi & 100 & $72.4 \pm 8.4$ & $1.45 \pm 0.17$ & N/A & 0.015 \\
Path & 50 & $49.0 \pm 0.0$ & $1.96 \pm 0.00$ & $49.0 \pm 0.0$ & 0.040 \\
Path & 75 & $74.0 \pm 0.0$ & $1.97 \pm 0.00$ & N/A & 0.027 \\
Path & 100 & $99.0 \pm 0.0$ & $1.98 \pm 0.00$ & N/A & 0.020 \\
\bottomrule
\end{tabular}}
\caption{Synthetic graph properties by type and size for the dataset used in the experiments about long range dependencies \ref{exp:long-range-deps} and scalability \ref{exp:compression-scalability} in the main paper.}
\label{tab:graph-stats-long-range}
\end{table}

\begin{table*}[t]
\centering
\small
\setlength{\tabcolsep}{5pt}
%\resizebox{\columnwidth}{!}{
\begin{tabular}{llcccc}
\toprule
Dataset & Size & Avg Edges & Avg Degree & Avg Diameter & Density \\
\midrule
\multirow{5}{*}{Cora}
& 10 & $12.4 \pm 2.4$ & $2.49 \pm 0.48$ & $3.0 \pm 0.7$ & 0.276 \\
& 20 & $29.7 \pm 6.7$ & $2.97 \pm 0.67$ & $3.9 \pm 0.9$ & 0.156 \\
& 30 & $46.3 \pm 9.7$ & $3.09 \pm 0.64$ & $4.5 \pm 1.1$ & 0.106 \\
& 40 & $61.8 \pm 14.1$ & $3.09 \pm 0.71$ & $4.6 \pm 1.1$ & 0.079 \\
& 50 & $78.7 \pm 10.1$ & $3.15 \pm 0.40$ & $5.1 \pm 1.2$ & 0.064 \\
\midrule
\multirow{5}{*}{Citeseer}
& 10 & $12.8 \pm 3.3$ & $2.56 \pm 0.67$ & $3.3 \pm 0.8$ & 0.285 \\
& 20 & $28.1 \pm 6.7$ & $2.81 \pm 0.67$ & $4.5 \pm 1.6$ & 0.148 \\
& 30 & $43.3 \pm 6.7$ & $2.89 \pm 0.45$ & $5.7 \pm 1.8$ & 0.100 \\
& 40 & $57.7 \pm 10.3$ & $2.89 \pm 0.51$ & $6.1 \pm 1.8$ & 0.074 \\
& 50 & $74.7 \pm 12.3$ & $2.99 \pm 0.49$ & $6.2 \pm 1.8$ & 0.061 \\
\midrule
\multirow{5}{*}{PubMed}
& 10 & $11.1 \pm 1.8$ & $2.21 \pm 0.36$ & $2.4 \pm 0.5$ & 0.246 \\
& 20 & $26.2 \pm 6.6$ & $2.62 \pm 0.66$ & $3.3 \pm 0.8$ & 0.138 \\
& 30 & $41.5 \pm 11.0$ & $2.77 \pm 0.74$ & $3.9 \pm 0.7$ & 0.095 \\
& 40 & $59.8 \pm 14.1$ & $2.99 \pm 0.70$ & $4.2 \pm 0.7$ & 0.077 \\
& 50 & $81.5 \pm 24.3$ & $3.26 \pm 0.97$ & $4.2 \pm 0.7$ & 0.067 \\
\midrule
\multirow{5}{*}{OGBN-ArXiv}
& 10 & $12.7 \pm 2.5$ & $2.54 \pm 0.51$ & $2.6 \pm 0.6$ & 0.282 \\
& 20 & $34.6 \pm 13.0$ & $3.46 \pm 1.30$ & $3.4 \pm 1.1$ & 0.182 \\
& 30 & $63.0 \pm 30.3$ & $4.20 \pm 2.02$ & $3.9 \pm 1.2$ & 0.145 \\
& 40 & $86.9 \pm 31.0$ & $4.35 \pm 1.55$ & $4.1 \pm 1.1$ & 0.111 \\
& 50 & $131.6 \pm 44.5$ & $5.26 \pm 1.78$ & $3.8 \pm 0.8$ & 0.107 \\
\bottomrule
\end{tabular}
%}
\caption{Graph statistics for sampled subgraphs used in node-classification experiments (Section~\ref{exp:predictive}). Values are reported as mean $\pm$ standard deviation over sampled subgraphs.}
\label{tab:graph-stats-predictive}
\end{table*}

\textbf{Synthetic Data Generation.} When generating synthetic data we can choose the specific number of nodes we want our graphs to have, the graph type, the task to be solved and how many of this graph-type-task combination we want to generate.
The specific statistics of the datasets used in section \ref{exp:intro}.
% More specifically, Table \ref{tab:graph-stats-reasoning} contains the statistics for the datasets used in experiments \ref{exp:quality}. Instead, Table \ref{tab:graph-stats-long-range} shows statistics for dataset used in \ref{exp:long-range-deps} and \ref{exp:compression-scalability}. 

\textbf{Real-World Datasets.} We consider four widely used real-world citation and academic graphs: Cora, Citeseer, PubMed, and OGBN-ArXiv. For each dataset, we sample induced subgraphs of varying sizes as described in Section~\ref{exp:predictive}, and compute structural statistics over the resulting instances.

Table~\ref{tab:graph-stats-predictive} reports the average number of edges, average degree, diameter, and density as a function of subgraph size. Across all datasets, we observe a consistent increase in the number of edges and average degree with graph size, while graph density decreases, reflecting the sparsity typical of real-world graphs. Diameters grow sublinearly with size, indicating that sampled subgraphs preserve small-world characteristics. OGBN-ArXiv subgraphs tend to exhibit higher average degree and variance compared to the other datasets, consistent with its denser and more heterogeneous global structure.

\section{Additional Results}
\label{app:additional_results}

This section reports additional experiments. They cover external graph-language baselines, lexical robustness of the color vocabulary, color-cardinality ablations, compressed prompting, local motif recognition, and prompt-length compression.

\subsection{Effects on Long Range Dependencies}
\label{exp:long-range-deps}
As we discussed in Section \ref{sec:method}, one of the major motivations of our design is to capture long-range dependencies (e.g., global structure). To demonstrate this, we evaluate performance on the maximum flow task where the source and target nodes are significantly farther apart. To achieve this, we construct 120 graphs with 50, 75, 100 nodes of path and ER graphs with $p=0.015$ (sparse, for high diameter), such that all the graphs' diameters $g_d$ are above a threshold, i.e., $g_d$ > 10. 

Table \ref{tab:longrange-merged} shows the results over different source-to-target distance ranges. Clearly, enriching the prompt with WL-labels not only enhances performance across all ranges, but also the improvement is larger as the distance increases. Indeed, encoding graphs from their non-Euclidean space to the textual sequence where position matters might make two nodes appear farther from each other in the sequence compared to their original shortest (path) distance in the graph. Therefore, attending to the WL-labels allows the LLM to compensate for the noisy signals from the positional-encoding module as its information is affected by the appearance position of the encoded node/edge tokens.

\begin{table*}[t]
\centering
\small
\setlength{\tabcolsep}{4pt}
\resizebox{.6\textwidth}{!}{\begin{tabular}{lcccccccc}
%\toprule
\toprule
\multicolumn{9}{c}{\textbf{Source-to-Target Distance Range}} \\
\midrule
 & \multicolumn{4}{c}{\textbf{Maximum Flow}} & \multicolumn{4}{c}{\textbf{Shortest Path}} \\
\cmidrule(lr){2-5}\cmidrule(lr){6-9}
Method 
& 10--15 & 16--25 & 26--40 & 41+ 
& 10--15 & 16--25 & 26--40 & 41+ \\
\midrule
\baselineOne{} 
& 3.1 & 0.0 & 5.0 & 0.0 
& \underline{6.5} & \textbf{6.2} & \textbf{20.0} & 0.0 \\

\cowl{} 
& 12.5 & \underline{9.1} & \underline{15.0} & \textbf{31.2} 
& 0.0 & 0.0 & \textbf{20.0} & \textbf{11.8} \\

\lowl{} 
& \underline{15.6} & \textbf{27.3} & \textbf{25.0} & 6.2 
& \underline{6.5} & \textbf{6.2} & 0.0 & \underline{5.9} \\

\ourmethod{} 
& \textbf{25.0} & \underline{9.1} & 5.0 & \underline{18.8} 
& \textbf{12.9} & \textbf{6.2} & \underline{6.7} & \underline{5.9} \\
\bottomrule
\end{tabular}}
\caption{Accuracy (\%) by source--target distance range for Maximum Flow and Shortest Path. Best values per column are in bold; second-best are underlined. Our \cowl variant generally outperforms the baselines. \ourmethod achieves superior performance on close range. \lowl performs best on medium to long ranges, and color-based variants achieve superior performance on the longest range.}
\label{tab:longrange-merged}
%\vspace{-4mm}
\end{table*}

\subsection{Scalability}
\label{exp:scalability}

We next evaluate whether the same structural advantage persists when graph size increases in hard reasoning tasks. Although modern LLMs support larger context windows, performance still deteriorates as prompts become longer and graph structure becomes harder to recover from a linearized serialization. We therefore study the same sparse high-diameter setting as in Section~\ref{exp:long-range-deps}, but now stratify results by graph size rather than by source--target distance. Table~\ref{tab:scalability-gpt35-both} reports accuracy for Maximum Flow and Shortest Path on graphs with 50, 75, and 100 nodes. Performance degrades substantially with graph size for all methods, confirming that larger graphs remain challenging even when they fit within context. However, WL-based variants mitigate this degradation. For Maximum Flow, the baseline is near failure across sizes, with aggregate accuracy only 2.53\%, whereas \lowl{} and the color-based variants raise aggregate performance to 16.46--17.72\%. For Shortest Path, the gains are smaller and less uniform, but \ourmethod{} attains the strongest aggregate performance at 8.86\%. These results mirror the long-range analysis in Section~\ref{exp:long-range-deps}: once direct structural recovery from the serialized graph becomes unreliable, WL-derived descriptors provide a more robust signal.

\begin{table*}[ht]
\centering
\small
\setlength{\tabcolsep}{5pt}
\resizebox{.6\textwidth}{!}{\begin{tabular}{lcccccccc}
\toprule
 & \multicolumn{4}{c}{\textbf{Maximum Flow}} & \multicolumn{4}{c}{\textbf{Shortest Path}} \\
\cmidrule(lr){2-5}\cmidrule(lr){6-9}
Method 
& 50 & 75 & 100 & All 
& 50 & 75 & 100 & All \\
\midrule
\baselineOne{} 
& 5.00 & 4.76 & 0.00 & 2.53 
& 5.00 & \textbf{14.29} & \underline{5.26} & \underline{7.59} \\

\cowl{} 
& \underline{45.00} & 4.76 & \underline{7.89} & \underline{16.46} 
& 0.00 & \underline{9.52} & \textbf{7.89} & 6.33 \\

\lowl{} 
& \textbf{50.00} & \textbf{9.52} & 5.26 & \textbf{17.72} 
& \underline{10.00} & 4.76 & 2.63 & 5.06 \\

\ourmethod{} 
& 30.00 & 4.76 & \textbf{15.79} & \underline{16.46} 
& \textbf{20.00} & \textbf{14.29} & 0.00 & \textbf{8.86} \\
\bottomrule
\end{tabular}}
\caption{Scalability of Maximum Flow and Shortest Path for different graph sizes. We report accuracy (\%). The best is in bold, and the second-best is underlined. Accuracy degrades rapidly with graph size for both tasks, with WL-based methods mitigating but not eliminating scaling failures. Note that the highest graph size used in \cite{fatemiTalkGraphEncoding2023} for TLG is 20.}
\label{tab:scalability-gpt35-both}
%\vspace{-2mm}
\end{table*}

To isolate this large-graph effect more directly, Table~\ref{tab:maxflow-50-100-er-path} compares Maximum Flow accuracy between 50-node and 100-node Erd\H{o}s--R\'enyi and path graphs. All WL-based variants remain substantially stronger than the baseline at both sizes. Moreover, the degradation from 50 to 100 nodes is much smaller for the color-based variants: \cowl{} drops by only 6.8 points, compared to 45.0 points for the baseline, 40.9 for \lowl{}, and 33.3 for \ourmethod{}. While \ourmethod{} achieves the highest accuracy at 50 nodes, \cowl{} is strongest at 100 nodes and exhibits the most stable scaling behavior. This suggests that the color-based realization of ordered WL structure is particularly helpful once graphs become large enough that the raw serialized edge list no longer reliably exposes the relevant global organization.

\begin{table}[ht]
\centering
\small
\setlength{\tabcolsep}{5pt}
\begin{tabular}{lccc}
\toprule
Variant & Acc.\ (50) & Acc.\ (100) & $\Delta$ \\
\midrule
\baselineOne{} & 45.0 & 0.0 & -45.0 \\
\lowl{}        & \underline{90.9} & 50.0 & -40.9 \\
\ourmethod{}   & \textbf{100.0} & \underline{66.7} & -33.3 \\
\cowl{}        & 81.8 & \textbf{75.0} & \textbf{-6.8} \\
\bottomrule
\end{tabular}
\caption{Maximum-flow accuracy (\%) on Erd\H{o}s--R\'enyi and path graphs with 50 and 100 nodes. $\Delta$ denotes the change in accuracy from 50 to 100 nodes; smaller absolute degradation is better. Best values per column are in bold; second-best are underlined.}
\label{tab:maxflow-50-100-er-path}
\end{table}

\subsection{Additional Baseline Comparisons}
\label{exp:comparison-with-baselines}

We further compare our method against recent graph-language baselines on node classification over Cora subgraphs. We follow the same evaluation setting as in Section~\ref{exp:predictive}: subgraphs of sizes 10, 20, 30, 40, and 50 nodes, with 30 sampled graphs per size, yielding 150 graph-task instances in total. For the external baselines, we follow the protocols provided by the original authors. GraphText is evaluated in inference mode using its official pipeline; LLaGA is evaluated with the released pretrained checkpoint; and OFA is trained and evaluated using the official codebase with the paper-reported supervised node-classification setting (100 epochs, learning rate $10^{-4}$, batch size 128). All methods are evaluated on the same sampled instances.

Table~\ref{tab:comparison-with-baselines} shows that all variants of our method outperform the external baselines as well as the TLG-A prompt baseline. In particular, \cowl{} achieves the best overall accuracy at 68.9\%, followed by \ourmethod{} at 66.4\% and \lowl{} at 63.6\%. GraphText performs similarly to TLG-A, while LLaGA and OFA are substantially weaker in this setting. These results indicate that the gains observed in Section~\ref{exp:predictive} are not merely due to applying an LLM to citation-graph subproblems, but are specifically associated with exposing ordered WL structure through the proposed textual encoding.

\begin{table}[t]
\centering
\small
\setlength{\tabcolsep}{5pt}
\begin{tabular}{lcc}
\toprule
Method & Type & Acc.\ (\%) \\
\midrule
LLaGA & Graph-LLM & 15.0 \\
GraphText & Graph-text & 49.7 \\
OFA & Graph model & 16.3 \\
\baselineOne{} & Prompt baseline & 49.0 \\
\lowl{} & Ours & 63.6 \\
\cowl{} & Ours & \textbf{68.9} \\
\ourmethod{} & Ours & \underline{66.4} \\

\bottomrule
\end{tabular}
\caption{Comparison with external baselines on Cora node classification. Accuracy is reported over all sampled instances. GraphText uses its official inference pipeline, LLaGA the released checkpoint, and OFA the official codebase with the paper-reported supervised node-classification setting.}
\label{tab:comparison-with-baselines}
\end{table}

\subsection{Local Motif Recognition}
\label{exp:local-motifs}

The graph-level triangle-counting task in Table~\ref{tab:algorithmic-tasks-gpt4o-agg} requires counting all triangles in the graph, which combines local motif detection with a global aggregation burden over the full serialized graph. Since our method enriches node-level representations through WL-derived structural descriptors, a weaker result on graph-level triangle counting does not necessarily imply a weakness in recognizing local triangular structure itself. To isolate this distinction, we introduce a three-node triangle-membership task that asks whether a queried node triple forms a triangle.

Table~\ref{tab:app-triangle-membership} shows that all WL-based variants improve over the TLG-A baseline on this localized formulation, with \ourmethod{} achieving the strongest performance at 68.3\% accuracy compared to 63.7\% for the baseline. The same trend is reflected in MAE and RMSE, where \ourmethod{} also performs best. This indicates that WL-based labels and colors do help the LLM recognize local motif structure when the task is aligned with the node-level information injected by the prompt.

\begin{table}[t]
\centering
\small
\setlength{\tabcolsep}{5pt}
\begin{tabular}{lccc}
\toprule
Variant & Acc.\ (\%) & MAE $\downarrow$ & RMSE $\downarrow$ \\
\midrule
\baselineOne{} & 63.7 & 0.363 & 0.603 \\
\cowl{} & 65.8 & 0.342 & 0.585 \\
\lowl{} & 66.0 & 0.340 & 0.583 \\
\ourmethod{} & \textbf{68.3} & \textbf{0.317} & \textbf{0.563} \\
\bottomrule
\end{tabular}
\caption{Triangle-membership results on a localized three-node query task. Unlike graph-level triangle counting, this formulation isolates local motif recognition by asking whether a queried triple forms a triangle.}
\label{tab:app-triangle-membership}
\end{table}

Taken together with the results of Section~\ref{exp:quality}, this experiment suggests that the weaker performance on graph-level triangle counting is primarily a task-formulation issue. For local motif recognition, where the relevant information is concentrated around a small queried set of nodes, the proposed node-level structural augmentation is beneficial.

\subsection{WL and Color Descriptor Statistics}
\label{exp:descriptor-stats}

To contextualize the design choices behind the WL-based descriptors, we report summary statistics of the refinement process on the main synthetic reasoning datasets. Table~\ref{tab:app-wl-size-stats} shows the number of WL iterations required for stabilization and the number of distinct structural labels and color descriptors produced at each graph size. Since cycle check, maximum flow, reachability, shortest path, and triangle counting are instantiated on the same generated graph pool, these statistics are identical across tasks for a fixed graph size and are therefore reported only once per size.

Two patterns are clear. First, the number of WL iterations needed for stabilization remains small across the main reasoning regime, typically around three to four iterations for graphs with 15--30 nodes. This supports the use of a small fixed WL refinement budget in these experiments. Second, the number of distinct WL labels grows steadily with graph size, which implies that larger graphs require either a larger descriptor vocabulary or a controlled level of collisions when labels are mapped to natural-language colors.

This behavior is also consistent with the color-cardinality ablation in Section~\ref{exp:importance-of-colors}. In the 50-node Cora node-classification setting used there, WL refinement requires $4.00 \pm 0.63$ iterations (range 3--5) and produces $31.80 \pm 8.89$ distinct WL labels per graph (range 18--46). With a small base color vocabulary, multiple WL labels must therefore map to the same color descriptor. Expanding the vocabulary with lightness modifiers reduces these collisions, but the results in Table~\ref{tab:color-cardinality-ablation} show that performance saturates once the effective collision factor is roughly between one and three. Overall, these statistics provide a descriptive account of the structural resolution induced by ordered WL refinement and help explain the empirical tradeoffs observed in the color-vocabulary ablations.

\begin{table*}[t]
\centering
\small
\setlength{\tabcolsep}{4pt}
\resizebox{.8\textwidth}{!}{\begin{tabular}{rcccc}
\toprule
Size & WL iters & WL labels & RGB colors & Color names \\
\midrule
10 & $2.78{\pm}0.79$ (2--4) & $6.67{\pm}3.13$ (2--10) & $6.67{\pm}3.13$ (2--10) & $4.67{\pm}1.63$ (2--6) \\
15 & $3.78{\pm}0.79$ (2--5) & $12.56{\pm}3.89$ (4--15) & $12.56{\pm}3.89$ (4--15) & $5.67{\pm}0.67$ (4--6) \\
20 & $3.78{\pm}0.92$ (3--6) & $18.22{\pm}2.53$ (12--20) & $18.22{\pm}2.53$ (12--20) & $6.00{\pm}0.00$ (6--6) \\
25 & $3.56{\pm}0.68$ (3--5) & $23.89{\pm}1.66$ (21--25) & $23.89{\pm}1.66$ (21--25) & $6.00{\pm}0.00$ (6--6) \\
30 & $3.44{\pm}0.68$ (3--5) & $29.67{\pm}0.67$ (28--30) & $29.67{\pm}0.67$ (28--30) & $6.00{\pm}0.00$ (6--6) \\
\bottomrule
\end{tabular}}
\caption{WL refinement statistics for the main synthetic reasoning datasets, aggregated by graph size. Values are reported as mean $\pm$ standard deviation, with min--max ranges in parentheses.}
\label{tab:app-wl-size-stats}
\end{table*}

\subsection{Full results on Graph Reasoning Tasks}

Tables~\ref{tab:reasoning-by-graphtype-gpt4o} and~\ref{tab:reasoning-by-graphtype-gpt35} report per-task reasoning performance stratified by graph type (BA, ER) and aggregated across graph types. These tables provide a finer-grained view of model behavior beyond the aggregated results discussed in Section~\ref{exp:quality}.

Across tasks, we observe substantial variability between BA and ER graphs, particularly for tasks requiring global reasoning such as Maximum Flow and Shortest Path. Performance on BA graphs is generally higher for cycle-related tasks, while ER graphs often exhibit increased difficulty for flow- and path-based reasoning, reflecting differences in graph structure and connectivity patterns. Aggregated results therefore mask non-trivial graph-type effects that are visible only at this granularity. While WL-enriched prompting improves performance for several tasks under \texttt{gpt-4o}, the same trends are less consistent for \texttt{gpt-3.5-turbo}, suggesting that the effectiveness of structural encodings depends on the underlying model’s reasoning capacity.

%Comparing models, \texttt{gpt-4o} consistently outperforms \texttt{gpt-3.5} across all tasks and metrics, with especially large gaps on more challenging reasoning tasks. While WL-enriched prompting improves performance for several tasks under \texttt{gpt-4o}, the same trends are less consistent for \texttt{gpt-3.5-turbo}, suggesting that the effectiveness of structural augmentations depends on the underlying model’s reasoning capacity.

Overall, these detailed results support the conclusions drawn in the main text while highlighting the sensitivity of graph reasoning performance to both graph type and model generation.

\begin{table*}[t]
\centering
\small
\setlength{\tabcolsep}{2.8pt}
\begin{tabular}{llcccccccccc}
\toprule
Task & Graph & \multicolumn{2}{c}{\baselineOne{}} & \multicolumn{2}{c}{\baselineTwo{}} & \multicolumn{2}{c}{\cowl{}} & \multicolumn{2}{c}{\lowl{}} & \multicolumn{2}{c}{\ourmethod{}} \\
\cmidrule(lr){3-4}\cmidrule(lr){5-6}\cmidrule(lr){7-8}\cmidrule(lr){9-10}\cmidrule(lr){11-12}
 &  & Acc.\ (\%) & MAE & Acc.\ (\%) & MAE & Acc.\ (\%) & MAE & Acc.\ (\%) & MAE & Acc.\ (\%) & MAE \\
\midrule
\multirow{3}{*}{Cycle Check}
& BA  & 100.00 & 0.0000 & 98.00 & 0.0200 & 100.00 & 0.0000 & 99.00 & 0.0100 & 100.00 & 0.0000 \\
& ER  & 79.00 & 0.2100 & 85.00 & 0.1500 & 85.00 & 0.1500 & 85.00 & 0.1500 & 86.00 & 0.1400 \\
& \textit{All}
& 89.50 & 0.1050
& \underline{91.50} & \underline{0.0850}
& \underline{92.50} & \underline{0.0750}
& 92.00 & 0.0800
& \textbf{93.00} & \textbf{0.0700} \\
\midrule
\multirow{3}{*}{Maximum Flow}
& BA  & 26.09 & 0.5233 & 34.78 & 0.4954 & 30.43 & 0.4863 & 17.39 & 0.6408 & 30.43 & 0.3564 \\
& ER  & 35.00 & 0.4870 & 37.00 & 0.4337 & 38.00 & 0.4330 & 41.00 & 0.4370 & 38.00 & 0.4573 \\
& \textit{All}
& 33.33 & 0.4938
& \underline{36.59} & \underline{0.4452}
& \underline{36.59} & \underline{0.4430}
& \underline{36.59} & 0.4751
& \textbf{36.59} & \textbf{0.4385} \\
\midrule
\multirow{3}{*}{Shortest Path}
& BA  & 73.91 & 0.0870 & 78.26 & 0.0725 & 82.61 & 0.0580 & 78.26 & 0.0725 & 73.91 & 0.0870 \\
& ER  & 91.00 & 0.0308 & 85.00 & 0.0447 & 90.00 & 0.0349 & 95.00 & 0.0137 & 90.00 & 0.0317 \\
& \textit{All}
& 87.80 & 0.0413
& 83.74 & 0.0499
& \underline{88.62} & \underline{0.0392}
& \textbf{91.87} & \textbf{0.0247}
& 86.99 & 0.0420 \\
\midrule
\multirow{3}{*}{Triangle Count}
& BA  & 11.00 & 0.4438 & 11.00 & 0.4701 & 12.00 & 0.4495 & 13.00 & 0.4660 & 13.00 & 0.3996 \\
& ER  & 21.00 & 0.4674 & 12.00 & 0.5764 & 16.00 & 0.4579 & 15.00 & 0.5509 & 16.00 & 0.5377 \\
& \textit{All}
& \textbf{16.00} & \textbf{0.4556}
& 11.50 & 0.5233
& 14.00 & \underline{0.4537}
& 14.00 & 0.5085
& \underline{14.50} & 0.4686 \\
\bottomrule
\end{tabular}
\caption{Per-task reasoning performance for \texttt{gpt-4o}, stratified by graph type (BA, ER) and aggregated (\textit{All}). Accuracy is higher-is-better; MAE is lower-is-better. Bold and underline indicate best and second-best methods, respectively, computed only on aggregated (\textit{All}) rows.}
\label{tab:reasoning-by-graphtype-gpt4o}
\end{table*}

\begin{table*}[t]
\centering
\small
\setlength{\tabcolsep}{2.8pt}
\begin{tabular}{llcccccccccc}
\toprule
Task & Graph & \multicolumn{2}{c}{\baselineOne{}} & \multicolumn{2}{c}{\baselineTwo{}} & \multicolumn{2}{c}{\cowl{}} & \multicolumn{2}{c}{\lowl{}} & \multicolumn{2}{c}{\ourmethod{}} \\
\cmidrule(lr){3-4}\cmidrule(lr){5-6}\cmidrule(lr){7-8}\cmidrule(lr){9-10}\cmidrule(lr){11-12}
 &  & Acc.\ (\%) & MAE & Acc.\ (\%) & MAE & Acc.\ (\%) & MAE & Acc.\ (\%) & MAE & Acc.\ (\%) & MAE \\
\midrule
\multirow{3}{*}{Cycle Check}
& BA & 66.00 & 0.3400 & 51.00 & 0.4900 & 66.00 & 0.3400 & 75.00 & 0.2500 & 72.00 & 0.2800 \\
& ER & 55.00 & 0.4500 & 52.00 & 0.4800 & 60.00 & 0.4000 & 54.00 & 0.4600 & 57.00 & 0.4300 \\
& \textit{All}
& 60.50 & 0.3950
& 51.50 & 0.4850
& \underline{63.00} & \underline{0.3700}
& \textbf{64.50} & \textbf{0.3550}
& \textbf{64.50} & \textbf{0.3550} \\
\midrule
\multirow{3}{*}{Maximum Flow}
& BA & 26.09 & 0.5746 & 21.74 & 0.6348 & 26.09 & 0.5971 & 30.43 & 0.5387 & 21.74 & 0.5522 \\
& ER & 15.00 & 0.7248 & 29.00 & 0.5377 & 10.00 & 0.7832 & 16.00 & 0.7032 & 11.00 & 0.7632 \\
& \textit{All}
& 17.07 & 0.6967
& \textbf{27.64} & \textbf{0.5558}
& 13.01 & 0.7484
& \underline{18.70} & \underline{0.6724}
& 13.01 & 0.7237 \\
\midrule
\multirow{3}{*}{Shortest Path}
& BA & 13.04 & 0.6232 & 8.70 & 0.6087 & 8.70 & 0.6377 & 8.70 & 0.6377 & 8.70 & 0.6377 \\
& ER & 14.00 & 0.6586 & 27.00 & 0.4347 & 9.00 & 0.7071 & 14.00 & 0.6653 & 10.00 & 0.6944 \\
& \textit{All}
& \underline{13.82} & \underline{0.6520}
& \textbf{23.58} & \textbf{0.4672}
& 8.94 & 0.6941
& 13.01 & 0.6602
& 9.76 & 0.6838 \\
\midrule
\multirow{3}{*}{Triangle Count}
& BA & 4.00 & 0.8557 & 4.00 & 0.8860 & 1.00 & 0.8616 & 6.00 & 0.8471 & 1.00 & 0.8818 \\
& ER & 10.00 & 0.7703 & 2.00 & 0.8686 & 5.00 & 0.7024 & 8.00 & 0.7930 & 7.00 & 0.6888 \\
& \textit{All}
& \textbf{7.00} & 0.8130
& 3.00 & 0.8773
& 3.00 & \textbf{0.7820}
& \textbf{7.00} & \underline{0.8201}
& \underline{4.00} & \underline{0.7853} \\
\bottomrule
\end{tabular}
\caption{Per-task reasoning performance for \texttt{gpt-3.5-turbo}, stratified by graph type (BA, ER) and aggregated (\textit{All}). Accuracy is higher-is-better; MAE is lower-is-better. Bold and underline indicate best and second-best methods, respectively, computed only on aggregated (\textit{All}) rows.}
\label{tab:reasoning-by-graphtype-gpt35}
\end{table*}
% \section{Example Appendix}
\label{sec:appendix}

\end{document}